 \documentclass[final,12pt]{clear2025} 
\usepackage{xfrac}
\usepackage{pythonhighlight}
\usepackage{outlines}
\usepackage{mathtools}
\usepackage{algorithm}
\usepackage{algpseudocodex}
\usepackage{tabstackengine}
\usepackage{multicol}
 \usepackage{vwcol}  
\usepackage{caption}
\DeclareFontFamily{U}{stix2bb}{}
\DeclareFontShape{U}{stix2bb}{m}{n} {<-> stix2-mathbb}{}
\NewDocumentCommand{\indicator}{}{\text{\usefont{U}{stix2bb}{m}{n}1}}
\setstackEOL{\\}
\newcommand{\Perp}{\!\perp\!\!\!\perp}
\title[UUMC SCM Generation]{Unitless Unrestricted Markov-Consistent SCM Generation: \\ Better Benchmark Datasets for Causal Discovery}
\usepackage{times}

\clearauthor{%
 \Name{Rebecca J. Herman} \Email{rebecca.herman@tu-dresden.de}\\
 \addr{Center for Scalable Data Analytics and Artificial Intelligence (ScaDS.AI), Dresden/Leipzig, Germany}, \\
 \addr{German Aerospace Center (DLR), Institute of Data Science, Jena, Germany}
 \AND 
 \Name{Jonas Wahl} \Email{jonas.wahl@dfki.de} \\
 \addr{German Research Centre for Artificial Intelligence (DFKI), Berlin, Germany}
 \AND
 \Name{Urmi Ninad} \Email{urmi.ninad@tu-berlin.de}\\
 \addr{Berlin Institute of Technology (TUB), Institute of Computer Engineering and Microelectronics, Germany}
 \AND
 \Name{Jakob Runge} \Email{jakob.runge@tu-dresden.de} \\
 \addr{Center for Scalable Data Analytics and Artificial Intelligence (ScaDS.AI), Dresden/Leipzig, Germany \\
 Berlin Institute of Technology (TUB), Institute of Computer Engineering and Microelectronics, Germany}
}

\begin{document}
\setlength{\abovedisplayskip}{4pt}
\setlength{\belowdisplayskip}{3pt}
\maketitle

\vspace{-18pt}
\begin{abstract}%
  Causal discovery aims to extract qualitative causal knowledge in the form of causal graphs from data. Because causal ground truth is rarely known in the real world, simulated data plays a vital role in evaluating the performance of the various causal discovery algorithms proposed in the literature. But recent work highlighted certain artifacts of commonly used data generation techniques for a standard class of structural causal models (SCM) that may be nonphysical, including var- and R2-sortability, where the variables' variance and coefficients of determination (R2) after regressing on all other variables, respectively, increase along the causal order. Some causal methods exploit such artifacts, leading to unrealistic expectations for their performance on real-world data. Some modifications have been proposed to remove these artifacts; notably, the internally-standardized structural causal model (iSCM) avoids varsortability and largely alleviates R2-sortability on sparse causal graphs, but exhibits a reversed R2-sortability pattern for denser graphs not featured in their work. We analyze which sortability patterns we expect to see in real data, and propose a method for drawing coefficients that we argue more effectively samples the space of SCMs.
  Finally, we propose a novel extension of our SCM generation method to the time series setting. These methods are implemented in the pythong package UUMCdata (\url{https://pypi.org/project/UUMCdata/}).
\end{abstract}

\vspace{-12pt}
\section{Introduction}\label{intro}
\vspace{-6pt}
Causal discovery aims to extract qualitative causal knowledge from observational (or imperfect or incomplete interventional) data \citep{Spirtes2000,peters2017elements,Pearl}, and a plethora of learning algorithms using different approaches exist for both the static \citep{CD-Review, SAT} and time series settings \citep{runge2019inferring, Runge-Nature, assaad2022survey, CAMPSVALLS20231}. Benchmark datasets with known causal ground truth are vital to spur method development and to compare and evaluate existing methods, but unfortunately, datasets with known ground truth are often limited to a few feature variables \citep[e.g. the bivariate T\"ubingen Cause-Effect pairs by][]{mooij2016distinguishing} because it is rare to find real multivariate data where the underlying causal mechanisms are known with high certainty \citep{runge2019inferring,brouillard2024landscape}. 
Hence, synthetic data have played a vital role in method development and evaluation, but if the simulated datasets do not reflect the properties of the real-world datasets on which the algorithms will be applied, then high performance during evaluation may not translate to the real world. 

There are typically two preliminary steps in data generation: generation of a causal graph, and assignment of causal functions and noise variables to form a Structural Causal Model (SCM). The details of both steps can affect the performance of causal discovery algorithms. It's generally unclear what constitutes the best approach for sampling causal graphs, and usually Erdös-R{\'e}nyi \citep[ER,][]{ER} and scale-free graphs \citep{SF} are tested in parallel. The first approach samples randomly over possible edges in graphs with distinguishable vertices, and the second mimics the distribution of graphs resulting from typical mechanisms of network growth. Unfortunately, neither approach is designed to reflect the datasets intentionally curated from natural phenomena by scientists, and both over-represent graph structures with more isomorphic forms. The graph isomorphism problem is computationally hard \citep{hardness}, and we do not address it here. Instead, we address the functional mechanisms for specified graphs, focusing on linear additive models with Gaussian noise, in which mechanisms consist of coefficients and noise variances.

A standard approach is to sample coefficients from some uniform distribution over a bounded interval such as $(-2, -0.5)\cup (0.5, 2)$ and set noise variances to 1 \citep[e.g.][]{NOTEARS}. However, recent work has highlighted certain artifacts that this method of choosing ``random" functional relationships leaves on the generated data in both the static \citep{Beware, R2} and time series settings \citep{Christopher}. 
Two types of artifacts have received much attention: varsortability \citep{Beware} and R2-sortability \citep{R2}, in which the sample variance of the variables or the coefficient of determination after regressing on the other variables, respectively, are related to the topological order of the variables. Many prominent benchmarking datasets, such as that used for the NeurIPS 2019 Causality 4 Climate (C4C) competition \citep{Causality4Climate}, are strongly characterized by varsortability, tempting some to conclude that ``effect variables tend to have larger marginal variance than their causal ancestors" and ``large regression coefficients may predict causal links better in practice than small p-values" \citep{winner}, but the validity of these claims rests strongly on whether it is reasonable to expect varsortability in the real world.

A couple alternative SCM generation methods have been proposed that attempt to avoid these artifacts \citep[e.g.][]{Mooij-DG, Squires-DG}, including the internally-standardized structural causal model \citep[iSCM;][]{iSCM}, which computationally standardizes every variable before generating the next. This removes varsortability and greatly reduces R2-sortability, as evidenced by numerical experiments on a certain subset of causal graphs. But their method of drawing causal parameters still limits the relative magnitudes of causal coefficients and noise terms, which may affect the absolute and relative performance of causal discovery methods in a nonphysical way.

We assert here that varsortability is nonphysical, but show that any data that is modeled by an SCM \textit{should} feature a mild R2-sortability tendency \textit{opposite} to that of standard data generation techniques. We argue that an unbiased SCM generation process should be \emph{unitless}, \emph{Markov-consistent} (remote structural properties are not associated with local functional parameters), and \emph{unrestricted} (the relative magnitudes of parameters can be arbitrarily large or small; see Section \ref{defs}), and propose a distribution from which to choose the coefficients and noise variances that satisfies these properties. A final contribution of our work is a novel extension of our SCM generation method to the time series setting, which enables benchmarking on realistic and theoretically grounded data generation models for data from time-dependent systems that ubiquitously occur across the sciences.

\section{Foundations}

\subsection{Linear Additive Models with Gaussian Noise} 

\pagebreak

\subsubsection{Structural Causal Models} \label{SCM}
\vspace{-3pt}

An acyclic linear additive Gaussian Structural Causal Model (SCM) with \(N\in \mathbb{N}\) variables can be fully specified by the set of structural assignments
\begin{align}\label{eq:SCM}\Big\{X_i := \sum_j {a_{ji}X_j+U_i}, \, U_i \sim \mathcal N(m_i, s_i)\end{align}
for \(i=1,\dots,N\), vectors \(\vec m \in \mathbb{R}^N \text{ and } \vec s\in  (\mathbb R^+)^N\), and upper diagonal matrix \( A= (a_{ji}) \in \mathbb{R}^{N \times N}\) (an \emph{intervention} on $X_j$ consists of changing the right-hand side of its structural assignment). We assign the indices \(i\) consistent with the topological order, so that each \textit{endogenous} variable \(X_i\) is directly affected by a subset of its \emph{predecessors} \(X_{<i}=\{X_j|j<i\}\) which we call its \emph{parents} \(pa(X_i)=\{X_j|a_{ji}\neq0\}\subseteq X_{<i}\). $X_i$ is affected by $X_j\in pa(X_i)$ via a linear causal effect \(a_{ji}\) where \(a_{ji}=0 \ \forall j\geq i\) enforces acyclicity, and by noise drawn from independent normal distributions $\mathcal N(m_i, s_i)$ with mean $m_i$ and standard deviation $s_i$. Any dependence between $X_i$ and its non-parent predecessors is destroyed by conditioning on the parents of \(X_i\): \(X_i \Perp X_{<i} \setminus pa(X_i) | pa(X_i)\), known as the \emph{local Markov property}. 
The SCM can be represented graphically by listing the \(\{X_i\}\) and connecting \(X_j\rightarrow X_i\) if and only if \(X_j \in pa(X_i) \Leftrightarrow a_{ji} \neq 0\), and \(X_j \in an(X_i)\subseteq X_{<i+1}\) is referred to as an \emph{ancestor} of \(X_i\) if there is a directed path $X_j \rightarrow \cdots \rightarrow X_i$ of any length. The graph $\mathcal G = (V, E)$ is notated as an ordered pair with a list of nodes $V = \{X_i | i=1,\dots,N\}$ and a matrix of oriented adjacencies $E = \indicator (a_{ji}\neq 0) \in \{0,1\}^{N\times N}$. 
The SCM may be standardized by subtracting the mean \(
\mu_i = \mathbb E[X_i] = \sum_j a_{ji}\mu_j+m_i\) and dividing by the standard deviation $\sigma_i = (\mathbb E[X_i^2] - \mu_i^2)^{\sfrac{1}{2}}$ for each variable $X_i$, yielding the \emph{Unitless SCM} (see Definition \ref{def:unitless})
\begin{align}\Big\{\hat{X}_i := \sum_j \hat{a}_{ji} \hat{X}_j + \hat{U}_i, \hat{U}_i \sim \mathcal N(0, \hat{s}_i)\end{align}
subject to the constraint
\(\sum_{jk}\hat{a}_{ji}\hat{a}_{ki}\hat{\rho}_{jk}+\hat{s}_i^2  = 1
\),
where \(\hat{a}_{ji} := \frac{\sigma_j}{\sigma_i}a_{ji}\) and \(\hat{s}_i := \frac{s_i}{\sigma_i}\) are unitless, and
\(
\hat{\rho}_{jk} = \mathbb E[\hat{X}_j\hat{X}_k] = \mathbb E\left[\frac{X_j-\mu_j}{\sigma_j}\frac{X_k-\mu_k}{\sigma_k}\right]  = \rho_{jk}
\)
are elements of the correlation matrix for \(\{X_i\}\).
\vspace{-3pt}
\subsubsection{Structural Vector Autoregressive Models} \label{SVAR SCM}
\vspace{-3pt}
For time series, we cannot assume that samples from $V$ are independent. Formally, we can include a node \(X_i(t)\) for every variable at every time, and allow a lagged causal effect $X_j(t-\tau) \rightarrow X_i(t)$ for any $\tau\in \mathbb N$. The resulting infinite graph is called a time series \citep{tsg} or full time graph \citep{peters2017elements}. Complete models often require contemporaneous dependencies as well (where $\tau=0$), and in practice, \(\tau \leq \tau_{\max}\) is bounded by some maximum lag \(\tau_{\max} \in \mathbb{Z}^+\). 

Without further assumptions, we would need a causal effect coefficient \(a_{ji,t}(\tau)\) for every pair \((X_j(t-\tau), X_i(t))\) of \(X_i\) and \(X_j\) at time \(t\) and lag $\tau\in 0,\dots,\tau_{\max}$, as well as means \(m_{i,t}\) and standard deviations \(s_{i,t}\) for the noise distributions $U_i(t)$ for every variable \(X_i\) at every time \(t\). Under the assumption of \textit{causal stationarity}, $a_{ji,t}(\tau)=a_{ji}(\tau)$, $m_{i,t}=m_i$, and \(s_{i,t}=s_i \;\forall i,j,t,\tau \), giving
\begin{align}\label{eq:SVAR}\Big\{X_i(t) := \sum_{j= 1}^N \sum_{\tau = 0}^{\tau_{\max}} {a_{ji}(\tau) X_j({t-\tau})+U_i(t)}, \; U_i \sim \mathcal N(m_i, s_i)\end{align}
This can be represented with a \textit{summary causal graph} with nodes $\{X_i|i\in 1,\dots,N\}$ and an edge $X_j \rightarrow X_i$ if $X_j(t-\tau)\in pa(X_i(t))$ for some $\tau \in 0,\dots,\tau_{\max}$. This process may contain causal effects from a variable to itself (when $i=j$), and though we restrict our analysis to acyclic time series graphs where \(a_{ji}(0)=0 \; \forall j\geq i\), \textit{unrolled} cross-dependencies that lead to a cyclic summary graph ($X_j(t-\tau)\rightarrow X_i(t)$ and $X_i(t-\nu)\rightarrow X_j(t)$, with $\tau>0, \;\nu \geq 0$) may appear. 

With linear additive causal effects $A(\tau) = (a_{ji})(\tau) \in \mathbb R^{N\times N \times \tau}$ and Gaussian noise, such a model would constitute a \textit{Structural Vector Autoregressive} model \citep[SVAR;][]{SVAR}   
\(\mathsf{A}\mathbf{X}(t) = \sum_{\tau=1}^{\tau_{\max}} \mathsf{A}A^*(\tau)\mathbf{X}(t-\tau) + \mathbf{U}(t)\)
where $\mathsf{A} = A(0)^{-1} + \mathbf{I}_N$ and $A^*(\tau) = \mathsf{A}^{-1}A(\tau)$ . 
This process is \textit{stable} if \(\det(\mathsf{A} - \sum_{\tau=1}^{\tau_{\max}}\mathsf{A}A^*(\tau)z^\tau) = \det(A(0)^{-1} + \mathbf{I}_N - \sum_{\tau=1}^{\tau_{\max}}A(\tau)z^\tau)=0\) has roots with magnitude less than 1 \citep[extended from][]{stability}. Stability is a sufficient condition for \textit{statistical stationarity}, meaning that the mean and (co)variance is equivalent conditioned on any time (but not on previous observations). The standardized form of the SVAR model is
\begin{align}\label{eq:uSVAR}\Big\{\hat{X}_i(t) = \sum_{j=1}^N\sum_{\tau=0}^{\tau_{\max}} \hat{a}_{ji}(\tau)\hat{X}_j(t-\tau) + \hat{U}_i(t),\; \hat{U}_i \sim \mathcal N(0, \hat{s}_i)\end{align}
where $\hat{a}_{ji}(\tau) = \frac{\sigma_j}{\sigma_i}a_{ji}(\tau)$ and $\hat{s}_i=\frac{s_i}{\sigma_i}$, 
constrained by \(\sum_{jk}\sum_{\tau\nu}\hat{a}_{ji}(\tau)\hat{a}_{ki}(\nu)\hat{\rho}_{jk}(\tau-\nu) + \hat{s}_i^2 = 1\) 
where 
\(
\hat{\rho}_{jk}(\omega) =  \mathbb E[\hat{X}_j(t-\omega)\hat{X}_k(t)] = \mathbb E\left[\frac{X_j(t-\omega) - \mu_j}{\sigma_j}\frac{X_k(t) - \mu_k}{\sigma_k}\right]
\)
are lagged correlations for $\{X_i\}$. 

\subsection{Benchmarking Data} \label{benchmark}
\subsubsection{Existing Benchmarks} \label{Existing}
Many real-world and synthetic causal discovery benchmarks focus on the bivariate cause-effect identification case. To name a few: the Tübingen Cause Effect Pairs \citep{mooij2016distinguishing} cover a large number of synthetic as well as real-world pairs of static and time series data; \citet{guyon2019cause} present results from a large benchmark competition, all of which is now freely available\footnote{\url{https://www.causality.inf.ethz.ch/cause-effect.php}}; and \citet{kaeding2023distinguishing} introduce a bivariate synthetic benchmark suite. But while there are some real-world multivariate static datasets used for benchmarking, such as those hosted at CMI\footnote{  \url{https://www.cmu.edu/dietrich/causality/projects/causal_learn_benchmarks/}} and the new customizable Causal Chamber \citep{CausalChamber}, they are few in number because it is quite challenging to obtain complete real-world ground truth causal knowledge, and thus synthetic data plays a more important role in benchmarking in the multivariate setting \citep{brouillard2024landscape}.  

The crucial choice to make when generating random SCMs and SVAR models is what joint distribution to use for the causal coefficients and the noise standard deviations. A simple and commonly-used approach in the static setting is to draw $a_{ji}$ and $s_i$ independently over some uniform probability distribution. \cite{NOTEARS} and many other studies use $\mathcal U([-2,-0.5]\cup [0.5, 2])$ for the causal coefficients and $1$ for the noise standard deviations (we call this `UVN' for `Unit Variance Noise'), but other studies make other choices; for example, \cite{TETRAD-experiment} use $\mathcal U(-1,1)$  for the causal coefficients and $\mathcal U(1,2)$ for the noise standard deviations. The linear-VAR datasets provided on \url{causeme.net} \citep{runge2019inferring} as part of the NeurIPS 2019 Causality 4 Climate (C4C) competition benchmarking data  \citep{Causality4Climate} adapt this approach to the time series setting including a test for stability, and \cite{TSgeneration} propose a framework for generating random time series benchmarking data with fewer assumptions about the data generation process.

Other approaches attempt to be more realistic by using real data and expert knowledge as a starting point. They often employ expert knowledge to constrain the causal graph and real data to fit the causal coefficients, and then generate simulated data from the resulting `ground truth' SCMs. Some static examples include \texttt{causalAssembly} \citep{causalAssembly}, which is based on data measurements taken from a manufacturing assembly line, and SynTReN \citep{SynTReN}, which mimics experimental gene expression data by randomly choosing from networks and functional relationships crafted by experts. Unfortunately, the applicability of such pseudo-real datasets is often limited to the field from which the real data was taken, and is further limited by the properties of the finite datasets on which the pseudo-real data was based. 

Pseudo-real time series data can be generated to mimic any user-provided real dataset via CausalTime \citep{CausalTime}, and existing pseudo-real benchmarking datasets mimicking climate and weather data can be found at \url{causeme.net}. CausalTime constructs the `ground truth' causal models by (1) using deep neural networks to fit a nonlinear AR model (NAR), (2) extracting possible tractable causal graphs using importance analysis or expert knowledge, and finally (3) extracting the part of the fitted NAR model consistent with the causal graph. CauseMe takes a similar approach: beginning with randomly-selected large climate simulations and variables, it (1) extracts major modes of variablility using PCA-Varimax \citep{Varimax}, (2) fits a VAR model to the data, and (3) uses the residuals to determine the magnitude of the noise. In addition to the limitations of static pseudo-real data, since both automated model-generation methods are fundamentally based on non-causal analysis, they may or may not be able to capture the desired real-world distribution of causal systems that regression-based methods would fail to fully characterize. 

\subsubsection{Sortability}\label{varsortability def}
\paragraph{Varsortability}
\cite{Beware} introduce the concept of \textit{varsortability} in static graphs, which means that variance increases along the topological order. \cite{Beware} define a sortability metric that ranges from 0 to 1, where values above 0.5 indicate varsortability and values below 0.5 indicate reverse-varsortability, and \citet{Christopher} extend their metric to the time series setting by applying it to summary graphs, excluding pairs of nodes that are cyclically related. \citeauthor{Beware} describe their metric as a ``fraction of directed paths," but algorithmically, they only count directed paths between the same pair of nodes separately if they have \textit{distinct lengths}, yielding a metric that is not directly related to directed paths nor to connected node pairs. We propose a simple modification to \citeauthor{Beware}'s algorithm that counts each pair of nodes exactly once (see Appendix \ref{alg:sort} for python code), changing sortability values quantitatively, but not qualitatively (see Figure \ref{triples} last subplot for an example using R2-scores instead of variance). \cite{Beware} observed that the static data generation techniques from Section \ref{Existing} produce strongly varsortable datasets, and we confirm this computationally (Figure \ref{fig:data_gen}a and c). \citet{Christopher} showed that generated time series datasets such as the C4C data are also often highly varsortable even when they are stable over time. Varsortability can be removed from generated data post-hoc via standardization. 

\paragraph{R2-sortability}
\citet{R2} found that the joint distributions for SCM parameters that produce varsortability in simulated datasets leave another artifact in the data that cannot be removed after the fact: increasing fractional cause-explained variance along the topological order. Though ground truth causal structure is needed to calculate cause-explained variance, they show that the coefficient of determination \citep[\(R^2\),][]{CoD}---which is measurable in the absence of causal knowledge---gives an upper-bound for the fraction of cause-explained variance and is an effective proxy for recovering the topological order from simulated static datasets (see Figure \ref{fig:data_gen}b and c). 

The extension of R2-sortability to the time series setting is less obvious than it is for varsortability; \cite{Christopher} extend the definition by focusing on the fraction of cause-explained variance due to \textit{distinct processes}, calculating $R^2$ for $X_i(t)$ based on $\{X_j(t-\tau)|j\neq i, \tau\in 0...\tau_{\max}\}$, and excluding the past of $X_i$. We refer to this definition as R2*-sortability. Strong varsortability appears to be associated with strong reverse-R2*-sortability in both real and simulated time series datasets \citep{Christopher}. The reversal may be because variables with large variance are likely to be affected more by their own past than by distinct processes when all causal edge weights are drawn from the same distribution. We assert that a better extension of R2-sortability to the time series setting should include the past of the entire system, thus approximating the fraction of cause-explained variance in the time series graph. Figure \ref{fig:R2ex} shows that this modified definition can revert the association so that varsortability is associated with R2-sortability, as in the static case.

\paragraph{Real Systems}
While simulated static and time series datasets are characterized by strong var- and R2-sortability patterns, \cite{Christopher} found that two real-world time series datasets---a river flow dataset where the causal relationships can reasonably be given by the locations of data collection relative to the river flow \citep{River}, and data from a physical Causal Chamber in which the researcher controls the causal relationships \citep{CausalChamber}---demonstrate opposite extremes in var- and R2*-sortability, suggesting that simulated sortability tendencies are unrealistic. 

\subsubsection{Avoiding Sortability}\label{remedies}
There have been a few attempts in the literature to alleviate var- and R2-sortability in generated static benchmark data. \cite{Mooij-DG} assign unit noise, then divide it and the causal coefficients $\vec{a_i}$ for each node $X_i$ by $(1+\sum_j a_{ji}^2)^{\sfrac{1}{2}}$: the standard deviation $X_i$ would have if its parents were independent (we call this `IPA' for `Independent Parents Assumption'). Datasets generated by IPA have weaker expected varsortability; but, in combination with the authors' choice of distribution for drawing $a_{ji}$, IPA produces reverse var- and R2-sortable datasets on average (Figure \ref{fig:data_gen}d-f). \cite{Squires-DG} account for covariance of the parents but enforce an even split between explainable variance and noise by generating data without noise, dividing the data and $\vec{a_i}$ by $\sqrt{2}$ times the sample standard deviation, and then adding noise with $s_i=\frac{\sqrt{2}}{2}$ (we call this `50-50'). 50-50 lacks any tendency for varsortability (Figure \ref{fig:data_gen}g and i), but produces a tendency toward reverse-R2-sortability that is just as strong as the original tendency toward R2-sortability in standard data generation techniques (Figure \ref{fig:data_gen}h and i).
\cite{iSCM} also standardize during data generation, but they scale $\vec{a_i}$ and $s_i$ by the full sample standard deviation $\sqrt{\frac{1}{P-1}\sum_p^P \big[U_j^{(p)2} + \big(\sum_j a_{ji}X_j^{(p)}\big)^2}\big]$, thus completely removing varsortability (Figure \ref{fig:data_gen}j). They name this hybrid SCM-standardization approach the `iSCM', but it can be viewed as a different sampling of the distribution of SCMs, like all approaches discussed in Section \ref{Existing} and here. The iSCM reduces R2-sortability enough that it is not noticeable on the sparse graphs featured in the main body of their paper, but a tendency toward reverse R2-sortability is still visible for denser graphs (Figure \ref{fig:data_gen}k and l). Finally, \cite{DaO} introduce the ``DAG Onion" method (DaO), which uniformly samples joint probability distributions consistent with a graph, then samples SCMs that could produce it. Contrary to the motivation given in the paper, DaO produces data that is strongly reverse varsortable (Figure \ref{fig:data_gen}m and o) and mildly (reverse) R2-sortable for sparse (dense) graphs (Figure \ref{fig:data_gen}n and o).

\section{Thought Experiments}
With so many different ways to sample SCMs, the question should not be how to remove any exploitable artifact, but rather, what artifacts do we expect to see, and how do we want to sample the space of possible SCMs? Do we expect sortability tendencies? Should all variables be equally noisy? Does it make sense to draw causal coefficients uniformly or to limit how large or small the coefficients can be? The following thought experiments explore these questions. 

\subsection{Toy Model: Sortability in Standard SCM Generation Techniques}\label{Toy}
\vspace{-3pt}
To gain intuition for the standard generation techniques described in Section \ref{benchmark}, we inspect a toy SCM generation process that always assigns unit noise and causal coefficients. Though not random, it retains the features that lead to sortability in commonly-used random SCM generation techniques. 
We begin by examining chains of different lengths. For a chain of length 2, such as $X \rightarrow Y$, our toy SCM generation method would produce $X$ with variance 1 -- all of it noise, and $Y$ with variance 2 -- half of which is noise and half of which is causally explainable. If we added a third variable to the chain, it would have a variance of 3, one third of which is noise and two thirds of which is explainable. Thus, this toy generation process yields SCMs where the total variance and fraction of causally-explainable variance for a variable reflects the number of \textit{ancestors} it has. This results in var- and R2-sortability because the maximum number of ancestors a variable can have in an acyclic SCM is limited by its place in the topological order. Although this relationship is not deterministic when coefficients are assigned randomly, it holds on average for SCM generation techniques using unit noise because all root nodes must have a variance of 1 (all of it noise), while any variable with ancestors must have a variance that is larger than 1 (because the noise is independent from the cause-explained variance) by an amount that is, on average, proportional to the variance of its parents. Thus, we note that causal discovery algorithms relying on the equal variance assumption \citep{DYNOTEARS_defense, equal_variance_1, equal_variance_2} implicitly assume sortability.

\vspace{-6pt}
\subsection{Modeling Choices and Marginalization: Number of non-Parent Ancestors}\label{marginalization}
\vspace{-3pt}
When appending variables to a chain, increasing variance along the topological order may at first appear to be a reasonable physical pattern, but it becomes clearly inappropriate when we recognize that the number of ancestors a variable has in an SCM is a \textit{modeling choice}, not a fundamental physical property of the system. Given an SCM, one can always define a smaller---but equally-valid---SCM over any subset of the original variables through \textit{marginalization}, which absorbs each excluded variable into the exogenous noise terms of its children. If we begin with a causally sufficient SCM (the noise terms are jointly independent) and every marginalized variable has at most one child, then the marginalized SCM will also be causally sufficient. Likewise, it is often possible to expand an SCM representing a particular physical system by including additional parents and ancestors. Is there a limit to how many predecessors one can add to a physically-meaningful SCM? This would mean reaching an indisputable ``root cause", which is not influenced by anything at all. Whether such a ``root cause" exists is a theological question rather than a scientific one, but personal beliefs cannot affect the conclusion that, in practice, there is no ``root cause" that can be included in an SCM, because those who believe in an absolute power also believe that it is unknowable.

We want the outputs of our SCM generation process to represent physical properties of real systems, not arbitrary modeling choices. Let's assume our example chain from Section \ref{Toy} represents some real physical system. How does our representation of the system change when we apply marginalization transformations? If we marginalize out \(X\), we are left with the causal chain \(Y\rightarrow Z\), where the variance of \(Y\) is still 2 and the variance of \(Z\) is still 3. In this example and in general, the total variance of a variable has no relationship to the number of parents or ancestors included in the SCM. But while the fraction of causally explainable variance of $Z$ is still $\sfrac{2}{3}$, removing $X$ renders $Y$ completely unexplainable. The difference is that $X$ was a \textit{parent} of $Y$, but only an \textit{ancestor} of $Z$, making it clear that the fraction of causally explainable variance in an SCM representing a real system is related to the number of \textit{parents} included in the scientist's model, rather than the number of \textit{ancestors}. A realistic SCM generation process should not produce variances that reflect the structure of the causal graph, and so the variables may as well be standardized; but it is acceptable---and expected---for probability distribution of the fraction of variance due to noise to be related to the number of \textit{parents} (but not ancestors) a variable has in the causal graph used to model the system. 

\vspace{-6pt}
\subsection{Remote Changes: Independence of Cause and Mechanism} \label{remote}

\begin{minipage}{0.8\linewidth}
Say some physical system can be represented by the causal graph given by the black arrows in the diagram to the right and the SCM with unit coefficients and noise. The variances of $Z$ and $X_2$ would be 1, the variance of \(X_1\) would be 2, and the variance of $Y$ would be 4, with three fourths of it explainable.  
\end{minipage} \hfill
\begin{minipage}{0.18\linewidth}
\vspace{-20pt}
\begin{equation*}
    \setstacktabbedgap{0pt}
    \Matrixstack{
         & &Z& & \\
         &\swarrow& &\color{lightgray} \searrow& \\
        X_1& & & &X_2 \\
         &\searrow& &\swarrow& \\
         & &Y& & 
    }    
\end{equation*}
\end{minipage}

Now, say we intervene on $X_2$, activating the gray edge by changing the structural assignment to $X_2 := -\frac{\sqrt{2}}{2}(Z+U_2)$  with $U_2\sim\mathcal N(0,1)$. The variance of $X_2$ is still $1$, but now it has a covariance with $X_1$ of $\mathbb E[(Z)(-\frac{\sqrt{2}}{2}Z)]=-\frac{\sqrt{2}}{2}$.
The physical process determining $Y$ should not change, but $Y$'s variance is now $\mathbb E[X_1^2] + 2\mathbb E[X_1X_2] + \mathbb E[X_2^2] + 1 = (2) + 2(-\frac{\sqrt{2}}{2}) + (1) +1 = 4-\sqrt{2}$, 
and its fraction of cause-explainable variance is $\frac{3-\sqrt{2}}{4-\sqrt{2}} = \frac{10-\sqrt{2}}{14}$. Though the actual functional parameters for \(Y\) do not respond to the correlation of $Y$'s parents, because the total variance has changed, the fraction of explainable variance changes, and all \textit{standardized} functional parameters must be scaled by the old standard deviation divided by the new standard deviation (see Section \ref{SCM}). Since the scalar is the same for all parameters, the \textit{relative} magnitudes and signs of the parameters (the ``mechanisms") will not change, and so we would not want the distribution of these relationships in randomly sampled standardized parameters for $Y$ to reflect the correlation between $X_1$ and $X_2$ (the ``causes"), consistent with the principle of separation of cause and mechanism.

\vspace{-6pt}
\subsection{Desired SCM-Generation Characteristics}\label{defs}
In Section \ref{marginalization}, we find that the variance of a variable is independent of modeled causal structure. Since its value depends on the chosen units and units are not specified, it is arbitrary.
In Section \ref{marginalization}, we find that the fraction of unexplained variance is independent of the number of non-parent ancestors, and in Section \ref{remote}, we find that the functional parameters are independent up to uniform scaling of the influence of remote structural features and functional parameters on the correlation of parents. We conclude that the distribution of functional parameters produced by an SCM generation process should depend only on the number of parents (local structure). This can be seen as a generalization of the local Markov property (which relates data to the SCM that generated it) to the SCM generation process itself. 
Finally, clearly, no physical limit on the ratio between causal coefficients and noise variance exists.

\newtheorem{dfn}{Definition}

\begin{dfn}
    An \emph{\textbf{SCM generation process}} over nodes $V$ and some function space with parameterization $\vec \Theta$ defines the distribution $\mathcal P(\vec \Theta | E) = \Pi_i \mathcal P(\vec \Theta_i|E,\vec \Theta\setminus\vec \Theta_i)$ for every (random) input graph $\mathcal G=(V,E)$, where $\vec \Theta_i$ is a vector of the parameters of the structural assignment for $X_i\in V$. It is...
\end{dfn}

\vspace{-6pt}
\begin{dfn}\label{def:unitless}
\emph{\textbf{unitless}} if $\forall E, X_i\in V$, $\mathbb E[X_i]=0,\mathbb E[X_i^2]=1$. 
\end{dfn}

\vspace{-6pt}
\begin{dfn}
\emph{\textbf{Markov-consistent}} if for any adjacency matrices $E,E'$ and nodes $X_i,X_j\in V$ with the same number of parents, $\exists c\in \mathbb R | c \neq 0 \wedge \mathcal P(\vec{\tilde \Theta}_i) = \mathcal P(c \odot \vec{\tilde \Theta}_j)$ where $\odot$ is the \emph{scaling product} for the parameterization $\vec \Theta$ and $\vec{\tilde \Theta}_i$ are elements of $\vec \Theta_i$ that are random given $E_i$ (see Appendix \ref{Appendix Scaling}).
\end{dfn}

\vspace{-6pt}
\begin{dfn}
\emph{\textbf{unrestricted}} if the distribution of the ratio of variance due to a single parent and to noise for each node has full support on $(0,\infty)$. 
\end{dfn}

\begin{multicols}{2}

\section{Method}
We propose Algorithm \ref{iid_alg}. If a variable $\hat{X}_i$ in a standardized system has $d_i$ independent parents, we may imagine its variance as a unit \(d_i\)-ball where each Cartesian coordinate represents the causal coefficient and standard deviation from one of the parents, and $1-r^2 < 1$ is the variance due to noise. We draw coefficients randomly and independently by sampling uniformly from this \(d_i\)-ball, which can be accomplished by drawing \(a_{ji}^{\prime\prime} \sim \mathcal N(0,1)\) iid from the standard normal distribution and scaling \(\vec{a_i}^{\prime\prime} = (a_{ji}^{\prime\prime})_j\) by \(r_i/||\vec{a_i}^{\prime\prime}||\), where \(r_i \sim \mathcal U^{\sfrac{1}{d_i}}(0,1)\) is drawn from the \(d_i\)\textsuperscript{th} root of the uniform distribution from 0 to 1 \citep{d-ball}. The fraction of the volume located near the outside of a \(d_i\)-ball increases with \(d_i\), so the expected fraction of variance due to noise is inversely related to the number of parents (average noisiness can be adjusted via the distribution for $r_i^{d_i}$), but recalling Section \ref{marginalization}, this is not concerning. To ensure unit variance with correlated parents, we standardize (because \(a_{ji}'=a_{ji}^{\prime\prime}=0\) when \(\hat{X}_j \notin pa(\hat{X}_i) \subseteq \hat{X}_{<i}\), we only need $\hat{X}_{<i}$ to do this).

\begin{algorithm}[H]\caption{UUMC SCM Generation}\label{iid_alg}

\begin{algorithmic}
    \Require a graph $\mathcal G=\left(\{\hat{X}_i\}, E\right)$ with topological order $i \in 1...N$ and adjacency matrix $E = (e_{ji}) \in {\{0,1\}}^{N\times N}$ with $N \in \mathbb N$
    \Ensure  $j\geq i \Rightarrow e_{ji}=0$
\end{algorithmic}
    $R = (\rho_{ji}) \gets \mathbb{I}_N$\; \\
    $\vec{s} = (\hat{s}_i) \gets \mathbf{1}_N$\; \\
    $\hat{A} = (\hat{a}_{ji}) \gets E$\; \\
    \For{$i \in 1...N$}{
        $d_i \gets \sum_j \vec{e}_i = \#pa(\hat{X}_i)$ \\ 
        \If{$d_i>0$}{
            draw $\vec{a}_i^{\prime\prime} \sim \mathcal N(0,1)^N$ \\
            draw $r_i \sim \mathcal U^{\sfrac{1}{d_i}}(0,1)$ \\
            $a_{ji}^{\prime\prime} \gets a_{ji}^{\prime\prime}e_{ji} \; \forall j \in 1...N$ \\
            $\vec{a}_i' \gets \frac{r_i}{\sum_j a_{ji}^{\prime\prime 2}}\vec{a}_i^{\prime\prime}$; $\;s_i' \gets \sqrt{1-r_i^2}$ \\
            $\sigma_i^{\prime} \gets  \sqrt{\vec{a_i'}^{T}R \vec{a_i'} + s_i^{\prime2}}$ \\
            $\hat{a}_{ji} \gets \frac{a_{ji}'}{\sigma_i'} \; \forall j\in 1...N$; $\;\hat{s}_i \gets \frac{s_i'}{\sigma_i'}$ \\
            $\rho_{ij} = \rho_{ji} \gets \sum_k \hat{a}_{ki}\rho_{jk} \; \forall j<i$
        }
    }
\Return{$\hat A$}
\end{algorithm}
\end{multicols}

\newtheorem{thm}{Theorem}
\begin{thm}
    Algorithm \ref{iid_alg} is unitless, unrestricted, and Markov-consistent.
\end{thm}

\begin{proof}
    Assume that $\hat X_{<i}$ are standardized. Then we have $\hat \mu_i = \sum_j \hat a_{ji} (0) + (0) = 0$ and $\hat \sigma_i^2 = \frac{1}{\sigma_i^{\prime 2}}\left( \mathbb E\left[\left(\sum_{j} a'_{ji}\hat X_j\right)^2\right] + (s_i')^2\right) = \frac{\sum_{jk}a'_{ji}a'_{ki}\rho_{jk} + s_i^{\prime2}}{\vec{a_i'}^{T}R \vec{a_i'} + s_i^2} = 1$, so Algorithm \ref{iid_alg} is unitless.  
    It is Markov-consistent because $\vec{a_i''}$ are drawn identically and independently, the distribution for $r$ depends only on $d_i$ (the number of parents), and the final modification is scaling all parameters by $\sigma_i'||\vec{a_i}^{''}||$. 
    Finally, Algorithm \ref{iid_alg} is unrestricted because the ratio of the variance due to an individual parent to that due to noise is proportional to $\frac{r_i^2}{1-r_i^2}$, which approaches 0 as $r\rightarrow0$ and approaches $\infty$ as $r\rightarrow1$.
\end{proof}
\vspace{-3pt}
\begin{remark}
    While the iSCM is also Markov-consistent, it is not unitless because it generates data that is standardized for a finite sample rather than in the infinite sample limit, and it is not unrestricted because it produces SCMs where the ratio of the magnitude of a causal coefficient $|\hat{a}_{ji}|$ to the noise standard deviation $\hat{s}_i$ for $\hat{X}_i$ is uniformly distributed on the interval $[0.5, 2]$. 
\end{remark}

\begin{figure}[b]
    \includegraphics[scale=0.6]{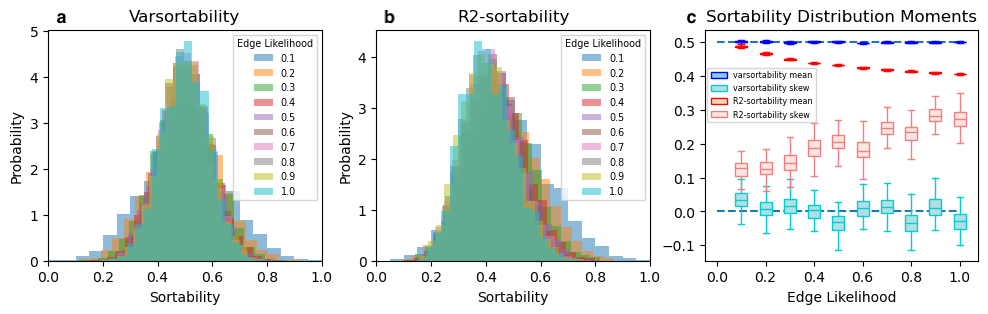}
    \vspace{-6pt}
    \caption{Sortability probability distribution functions for static ER graphs with 20 nodes and varying edge probabilities, based on 5000 randomly generated SCMs and 100 data samples.}
    \label{SortabilityStats}
\end{figure}

\vspace{-6pt}
\section{Experiments}
\vspace{-3pt}
\subsection{Sortability Properties}
\vspace{-3pt}
\begin{figure}
\begin{minipage}{0.61\linewidth}
    \includegraphics[scale=0.55]{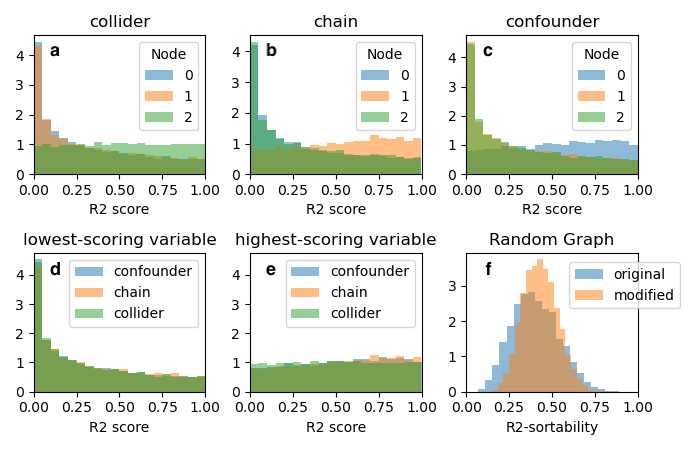}%
    \vspace{-6pt}
     \captionof{figure}{\textbf{a-c} R2 scores for nodes in the three kinds of unshielded triples, labeled by topological order, based on 10000 random SCMs and 500 samples. \textbf{d} Highest and \textbf{e} lowest scoring node from each triple. \textbf{f} R2-sortability for ER graphs with 20 nodes and edge probability 0.5, based on 5000 random SCMs and 100 samples, according to the original (blue) and modified (orange) sortability metrics (see Section \ref{varsortability def}).} \label{triples}
\end{minipage} \hfill
\begin{minipage}{0.35\linewidth}
    \vspace{18pt}
    \includegraphics[scale=0.6]{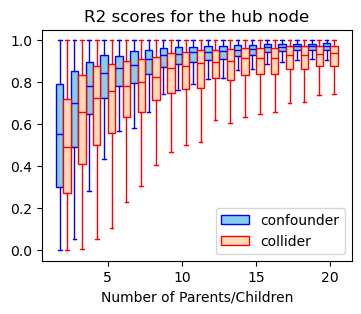}%
    \vspace{18pt}
     \captionof{figure}{R2 scores for the hub node of colliders and confounders with different numbers of parents/children, based on 500 randomly-generated static colliders and confounders of each size and 500 simulated data points.}
     \label{hub}
\end{minipage}
\end{figure}

As intended, our SCM generation method results in an expectation of neutral varsortability for ER graphs (Figure \ref{SortabilityStats}a and c). Though the SCM is constructed to be standardized in the infinite sample limit, finite sample size still leads to a wide range of simulated varsortability values. Likewise, this data-generation method produces SCMs with a realistically-wide range of R2-sortability (Figure \ref{SortabilityStats}b). This is because, in the absence of varsortability, R2 scores are highest for the node with the most neighbors, regardless of their causal order. We confirm this by examining the R2 scores of nodes in the three types of unshielded triples in Figure \ref{triples}a-c. The nodes are labeled by topological order; the nodes with 2 neighbors include node 2 for the collider, node 1 for the chain, and node 0 for the confounder. The expected R2 scores for nodes with only one neighbor are low in each case, while the expected R2 score for the node with 2 neighbors is higher, favoring nodes at different positions in the topological order depending on the connectivity of the graph. 

When confounders and colliders occur with equal frequency, one might expect the R2-sortability tendencies for these different structures to exactly cancel out. Instead, we see a positive skew (Figures \ref{triples}f and \ref{SortabilityStats}b), meaning our data is \textit{reverse} R2-sortable on average, similar to the iSCM and IPA approaches (Figure \ref{fig:data_gen}e and k). This holds for random graphs of all densities (Figure \ref{SortabilityStats}b-c). Should we expect such an artifact? Analytically, yes. The expected R2 score is not the same for the middle node of different types of triples (proof in Appendix \ref{Appendix R2}). For a collider such as \(\hat{A}\rightarrow \hat{B} \leftarrow \hat{C}\), the R2 score for \(\hat{B}\) is upper-bounded by \(1-\hat{s}_b^2\), because there is no information in \(\hat{A}\) or \(\hat{C}\) that could possibly help predict \(\hat{U}_B\). However, for a confounded triple such as \(\hat{A} \leftarrow \hat{B} \rightarrow \hat{C}\), the R2 score is actually lower-bounded by \(1-\min{(\hat{s}_A^2, \hat{s}_C^2)}\), and can get arbitrarily close to 1. This is because all information from the parent influences every child. While that information is obscured by the noise terms of the children, these noise terms are mutually independent, so they partially cancel each other during the attempted reconstruction of the parent node. 

We confirm this computationally: while R2-scores are equally distributed for nodes with one neighbor in all types of unshielded triples (Figure \ref{triples}d), the R2 score for the middle, or \textit{hub}, node of the collider (green) is on average slightly lower than that of the other structures (Figure \ref{triples}e). This pattern holds for larger structures as well. Naturally by the law of large numbers, as we increase the number of children for a confounding hub node, its expected R2 score increases (Figure \ref{hub}, blue) -- a pattern which should hold regardless of the data generation strategy. On the other hand, the behavior of the R2 score of a `colliding' hub node as we increase the number of parents (red) reflects our choice to reduce the fraction of variance due to noise for variables with more parents. 
The similarity in the relationship of R2 scores to the number of neighbors for hub nodes from these different structures gives us confidence that this was a realistic choice. Despite the similarity, the upper-bound for the collider means that it always lags the R2-score for the hub of a confounder. 

\noindent
\begin{minipage}{0.49\linewidth}
\setlength{\parindent}{2em}
\vspace{-12pt}
\subsection{Causal Discovery Evaluation}
In Figure \ref{fig:CD}, we examine the impact of data generation strategy on the apparent performance of various static causal discovery algorithms (mostly as implemented in gCastle \citep{gcastle}). PC(-stable) are constraint-based methods relying on conditional independence testing \citep{PC} and the \emph{faithfulness assumption} (that connections in the causal graph manifest as statistical relationships in the data) to discover partially-oriented equivalence classes. The other methods learn fully-oriented graphs; Greedy Equivalence Search (GES) is score-based \citep{GES}, NOTEARS is gradient-based \citep{NOTEARS} like the winner of the C4C competition \citep{winner}, and Best Order  Score Search (BOSS) is a 

\end{minipage}\hfill
\begin{minipage}{0.49\linewidth}
\begin{figure}[H]
    \centering
    \includegraphics[width=\linewidth]{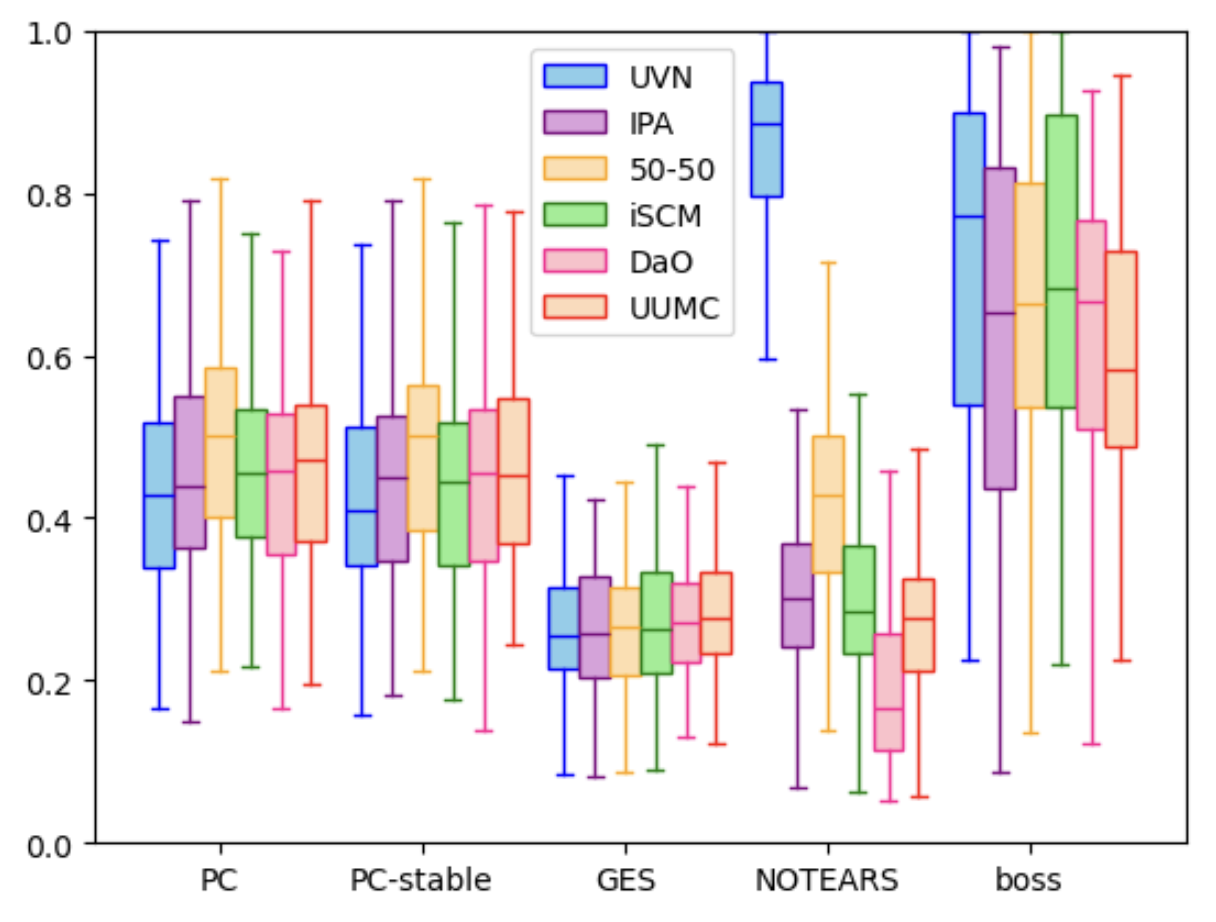}
    \vspace{-28pt}
    \caption{F1-score comparison of static causal discovery algorithms on benchmark data generated in different ways (see Section \ref{benchmark}) based on 100 SCMs with 10 variables and 500 data samples. ``UUMC" is the method proposed here.\\}
    \label{fig:CD}
\end{figure}
\end{minipage} 

\vspace{-12pt}
\noindent permutation-based algorithm that uses a score based on conditional independence constraints and relaxes the faithfulness assumption by relying more strongly on causal sufficiency. 

In blue, we show the performance of these algorithms on data generated in the standard way (denoted UVN for unit-variance-noise, see Section \ref{Existing}) according to gcastle's directed F1 score. NOTEARS, which was introduced accompanied by an evaluation on UVN data, far outperforms the other algorithms on this dataset, followed closely only by BOSS. Note that the apparent performance of PC(-stable) is limited because the F1 score is not implemented for equivalence classes.

In purple, yellow, green, and pink, we present the performance of these algorithms on data generated according to the four proposals for addressing sortability artifacts discussed in Section \ref{remedies}. 
By far the most dramatic change is the reduced performance of NOTEARS for any of these approaches, all of which reduce varsortability, and especially for DaO (pink), which exhibits reverse varsortability (Figure \ref{fig:data_gen}). Though not nearly as dramatic, the performance of BOSS is also markedly reduced for all modified data generation approaches. Within the modified approaches, the performances of score-search algorithms like GES and BOSS remain relatively constant, but the constraint-based algorithms perform noticeably better on the 50-50 data (yellow), likely because this approach ensures that causal effects are not overshadowed by noise while avoiding computational faithfulness violations due to near-determinism \citep[see][for a discussion of the relevant assumptions]{runge2018causal}. Interestingly, NOTEARS also performs better on the 50-50 data, suggesting that it is more effective when the true fraction of explainable variance is similar for all variables.

Finally, the performance of these algorithms on data from UUMC SCMs generated using Algorithm \ref{iid_alg} is shown in red. Compared to the iSCM, our approach yields a marginally higher performance for the PC algorithms and a marginally lower performance for NOTEARS, slightly increasing the apparent marginal performance of constraint-based relative to gradient-based algorithms. Furthermore, while BOSS has the highest performance of all the algorithms on data that is not varsortable, it performs notably worse on the UUMC data than any of the other modified approaches, materially reducing the apparent performance improvement provided by BOSS.

\noindent
\begin{minipage}{0.66\linewidth}
\setlength{\parindent}{2em}

\noindent
\section{Extension to Time Series Data}\label{TS Exp Intro}

Extending the method to the time series setting is non-trivial. Even if we insist that our method reduces to the static case when $\tau_{\max}=0$ thus forcing all dependencies to be contemporaneous, the extension to lagged (auto-)dependencies is not obvious. Note, for instance, that the discrete auto-dependencies in an SVAR model (equation \ref{eq:SVAR}) are naturally unitless---$a_{ii}(\tau) = \hat{a}_{ii}(\tau)$, where $\hat{a}_{ii}(\tau)$ are the auto-dependencies in the unitless SVAR model defined by eq.~\ref{eq:uSVAR}---while other lagged dependencies ($a_{ji}(\tau), \; j\neq i, 0<\tau\leq\tau_{\max}$) are unit-dependent. Furthermore, we can no longer simply use coefficients from the causal model to represent the contribution of a parent process to its child, since the parent may affect the child process via multiple lags.

\end{minipage}
\hfill
\begin{minipage}{0.3\linewidth}
\begin{figure}[H]
    \includegraphics[width=\linewidth]{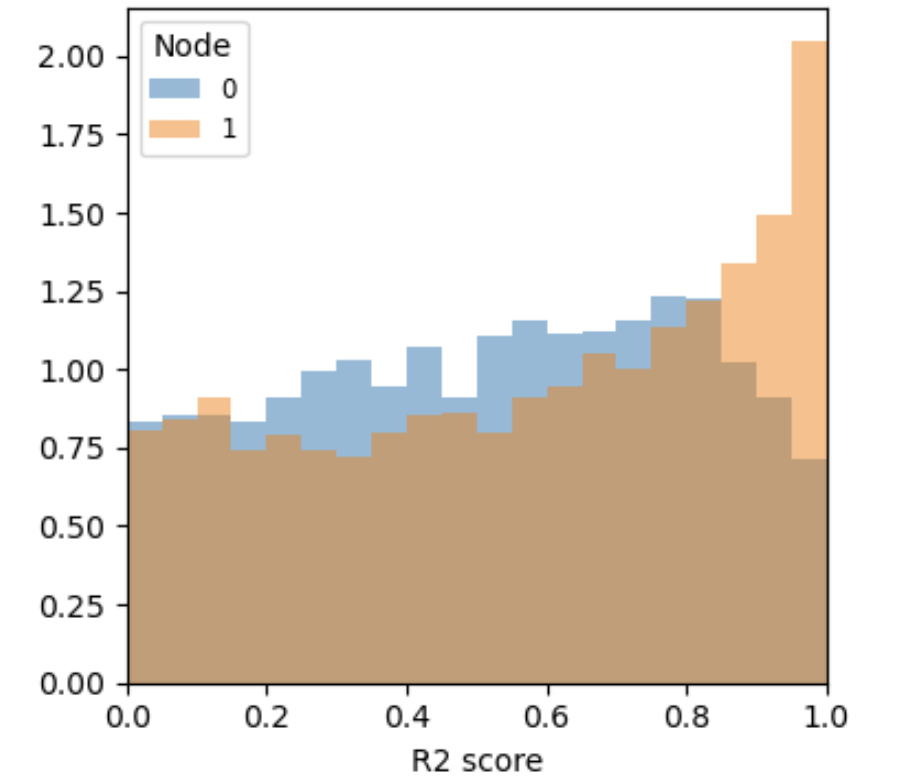}
    \caption{R2-scores for random SCMs corresponding to the time series graph given by $\hat{X}_0(t-1)\rightarrow \hat{X}_0(t)\rightarrow \hat{X}_1(t)\rightarrow \hat{X}_1(t+1)$.}
    \label{fig:TS}
\end{figure}
\end{minipage}

In Appendix \ref{Time Series Appendix}, we propose a theoretical foundation for drawing causal parameters based on a comparison to continuous processes. We approximate the contribution of a parent process with the sum of the causal coefficients at all lags, inspired by the direct transfer function at frequency 0 
\citep[see Appendix \ref{mult} and][]{Nicolas}. Auto-dependence is drawn from $\mathcal U(0,1)$, the fraction of variance due to noise is drawn relative to the part of the total variance not due to auto-dependence.

 Both auto- and cross-dependencies are adjusted to respond to a random sampling rate, inspired by discrete sampling of Ornstein-Uhlenbeck processes (see Appendix \ref{OU}), and we show that hidden confounding from subsampling can be bounded based on the sampling frequency. Finally, we solve a system of equations yielding the (auto-)covariances and scaling parameters for noise and the cross-coefficients to achieve standardized time series. The full method is given in Algorithm \ref{ts_alg}. It is designed to produce datasets with unit variance, but R2-sortability tendencies may emerge. 

Figure \ref{fig:TS} shows that for a simple 2-variable SVAR process with lag-1 auto-dependencies and $\hat{X}_0(t)\rightarrow \hat{X}_1(t)$, the distribution of R2 scores for the source node (blue) differs from that of the target node (orange). On the whole, time series datasets generated this way tend toward R2-sortability.

\vspace{-12pt}
\section{Discussion}
We examine data generation techniques that aim to avoid var- and R2-sortability, and evaluate their success in removing these signatures of the underlying causal structure. Notably, the Dag-Onion method---which is meant to naturally avoid nonphysical artifacts and ensure simulation of data with all possible correlation structures---results in strong reverse varsortability and different R2-sortability tendencies for graphs of different densities. Also, while the iSCM eliminates varsortability, it demonstrates mild reverse R2-sortability that is stronger for denser graphs. Is this a problem? 

We propose a method for sampling Structural Causal Models (SCMs) that is based on theoretical arguments that appropriately distinguish between physical properties of a system and arbitrary modeling choices such as the number of included ancestors or the choice of units. We find that its R2-sortability properties are similar to those of the iSCM. For static graphs, we demonstrate both theoretically and computationally that a node's R2-scores on average reflect its connectivity, but that they are a poor indicator for the topological order. Thus, we assert that while a tendency for varsortability is a nonphysical artifact of standard data generation processes, R2-scores do reflect aspects of the causal structure in any system that can be represented by an SCM. To the extent that R2-scores are related to the topological order, \textit{reverse} R2-sortability is more likely than R2-sortability, because the fact that SCMs are not deterministic limits how well one can predict a child, but multiple children can serve as independent observations of a parent, producing arbitrarily good reconstructions of the parent in the limit of many children. 

Despite having similar sortability properties, when deployed for evaluation of causal discovery algorithms, our approach results in mild improvements for constraint-based over gradient-based algorithms and a dramatic drop in the performance of the permutation-based algorithm BOSS when compared to evaluations using the iSCM. This provides evidence that artifacts other than var- and R2-sortability, such as restrictions on the ratios of functional parameters, can affect the absolute and relative performances of causal discovery algorithms even when they do not explicitly exploit these artifacts, and highlights the need for intentional SCM sampling methods like that proposed here, rather than attempting to remove one artifact at a time from standard sampling methods.

Finally, we take a first step toward intentional sampling of SVAR models. Surprisingly, we find that our method produces time series datasets that are slightly R2-sortable on average. 
Our method accounts for additional noise due to subsampling, but neglects the covariance of the noise terms in multivariate systems. Though we bound the magnitude of hidden confounding in terms of the sampling rate, there is a need to rigorously address the issue of pervasive hidden-confounding that is a necessary consequence of subsampling continuous systems.

\vspace{-12pt}
\acks{RH, JW, UN and JR received funding from the European Research Council (ERC) Starting Grant CausalEarth under the European Union’s Horizon 2020 research and innovation program (Grant Agreement No. 948112). JW received support from the German Federal Ministry of Education and Research (BMBF) as part of the project MAC-MERLin (Grant Agreement No. 01IW24007). This work used resources of the Deutsches Klimarechenzentrum (DKRZ) granted by its Scientific Steering Committee (WLA) under project ID 1083.}

\pagebreak
\bibliography{bibliography}

@String{Computing = "Computing" }

@String{Springer = "Springer-Verlag" }

@misc{Christopher,
  title={Sortability of Time Series Data},
  author={Lohse, Christopher and Wahl, Jonas},
  howpublished={arXiv preprint},
  year={2024},
  url={https://doi.org/10.48550/arXiv.2407.13313},
  note = "Contribution to the Causal Inference for Time Series Data Workshop at the 40th Conference on Uncertainty in Artificial Intelligence",
    }

@article{SynTReN,
  title={SynTReN: a generator of synthetic gene expression data for design and analysis of structure learning algorithms},
  author={Van den Bulcke, Tim and Van Leemput, Koenraad and Naudts, Bart and van Remortel, Piet and Ma, Hongwu and Verschoren, Alain and De Moor, Bart and Marchal, Kathleen},
  journal={BMC bioinformatics},
  volume={7},
  pages={1--12},
  year={2006},
  publisher={Springer},
  url={https://link.springer.com/article/10.1186/1471-2105-7-43},
  doi={10.1186/1471-2105-7-43}
}

@InProceedings{causalAssembly,
  title = 	 {$\texttt{causalAssembly}$: Generating Realistic Production Data for Benchmarking Causal Discovery},
  author =       {G\"obler, Konstantin and Windisch, Tobias and Drton, Mathias and Pychynski, Tim and Roth, Martin and Sonntag, Steffen},
  booktitle = 	 {Proceedings of the Third Conference on Causal Learning and Reasoning},
  pages = 	 {609--642},
  year = 	 {2024},
  editor = 	 {Locatello, Francesco and Didelez, Vanessa},
  volume = 	 {236},
  series = 	 {Proceedings of Machine Learning Research},
  month = 	 {01--03 Apr},
  publisher =    {PMLR},
  url = 	 {https://proceedings.mlr.press/v236/gobler24a.html}
}

@inproceedings{TETRAD-experiment,
 author = {Andrews, Bryan and Ramsey, Joseph and Sanchez Romero, Ruben and Camchong, Jazmin and Kummerfeld, Erich},
 booktitle = {Advances in Neural Information Processing Systems},
 editor = {A. Oh and T. Naumann and A. Globerson and K. Saenko and M. Hardt and S. Levine},
 pages = {63945--63956},
 publisher = {Curran Associates, Inc.},
 title = {Fast Scalable and Accurate Discovery of DAGs Using the Best Order Score Search and Grow Shrink Trees},
 url = {https://proceedings.neurips.cc/paper_files/paper/2023/file/c9cde817d04811ba28e44071bd9f76a5-Paper-Conference.pdf},
 volume = {36},
 year = {2023}
}

@inproceedings{NOTEARS,
 author = {Zheng, Xun and Aragam, Bryon and Ravikumar, Pradeep K and Xing, Eric P},
 booktitle = {Advances in Neural Information Processing Systems},
 editor = {S. Bengio and H. Wallach and H. Larochelle and K. Grauman and N. Cesa-Bianchi and R. Garnett},
 pages = {},
 publisher = {Curran Associates, Inc.},
 title = {DAGs with NO TEARS: Continuous Optimization for Structure Learning},
 url = {https://proceedings.neurips.cc/paper_files/paper/2018/file/e347c51419ffb23ca3fd5050202f9c3d-Paper.pdf},
 volume = {31},
 year = {2018}
}

@InProceedings{winner,
  title = 	 {Causal structure learning from time series: Large regression coefficients may predict causal links better in practice than small p-values},
  author =       {Weichwald, Sebastian and Jakobsen, Martin E. and Mogensen, Phillip B. and Petersen, Lasse and Thams, Nikolaj and Varando, Gherardo},
  booktitle = 	 {Proceedings of the NeurIPS 2019 Competition and Demonstration Track},
  pages = 	 {27--36},
  year = 	 {2020},
  editor = 	 {Escalante, Hugo Jair and Hadsell, Raia},
  volume = 	 {123},
  series = 	 {Proceedings of Machine Learning Research},
  month = 	 {08--14 Dec},
  publisher =    {PMLR},
  url = 	 {https://proceedings.mlr.press/v123/weichwald20a.html}
}

@inproceedings{R2,
 author = {Reisach, Alexander and Tami, Myriam and Seiler, Christof and Chambaz, Antoine and Weichwald, Sebastian},
 booktitle = {Advances in Neural Information Processing Systems},
 editor = {A. Oh and T. Naumann and A. Globerson and K. Saenko and M. Hardt and S. Levine},
 pages = {785--807},
 publisher = {Curran Associates, Inc.},
 title = {A Scale-Invariant Sorting Criterion to Find a Causal Order in Additive Noise Models},
 url = {https://proceedings.neurips.cc/paper_files/paper/2023/file/027e86facfe7c1ea52ca1fca7bc1402b-Paper-Conference.pdf},
 volume = {36},
 year = {2023}
}

@InProceedings{DYNOTEARS_defense,
  title = 	 {Structure Learning with Continuous Optimization: A Sober Look and Beyond},
  author =       {Ng, Ignavier and Huang, Biwei and Zhang, Kun},
  booktitle = 	 {Proceedings of the Third Conference on Causal Learning and Reasoning},
  pages = 	 {71--105},
  year = 	 {2024},
  editor = 	 {Locatello, Francesco and Didelez, Vanessa},
  volume = 	 {236},
  series = 	 {Proceedings of Machine Learning Research},
  month = 	 {01--03 Apr},
  publisher =    {PMLR},
  url = 	 {https://proceedings.mlr.press/v236/ng24a.html}
}

@inproceedings{Beware,
 author = {Reisach, Alexander and Seiler, Christof and Weichwald, Sebastian},
 booktitle = {Advances in Neural Information Processing Systems},
 editor = {M. Ranzato and A. Beygelzimer and Y. Dauphin and P.S. Liang and J. Wortman Vaughan},
 pages = {27772--27784},
 publisher = {Curran Associates, Inc.},
 title = {Beware of the Simulated DAG! Causal Discovery Benchmarks May Be Easy to Game},
 url = {https://proceedings.neurips.cc/paper_files/paper/2021/file/e987eff4a7c7b7e580d659feb6f60c1a-Paper.pdf},
 volume = {34},
 year = {2021}
}

@article{Varimax, 
    title={A spatiotemporal stochastic climate model for benchmarking causal discovery methods for teleconnections}, 
    volume={1}, 
    DOI={10.1017/eds.2022.11}, 
    journal={Environmental Data Science}, 
    author={Tibau, Xavier-Andoni and Reimers, Christian and Gerhardus, Andreas and Denzler, Joachim and Eyring, Veronika and Runge, Jakob}, 
    year={2022}, 
    pages={e12},
    doi={10.1017/eds.2022.11},
    url={https://doi.org/10.1017/eds.2022.11}
}

@misc{CausalTime,
      title={CausalTime: Realistically Generated Time-series for Benchmarking of Causal Discovery}, 
      author={Yuxiao Cheng and Ziqian Wang and Tingxiong Xiao and Qin Zhong and Jinli Suo and Kunlun He},
      year={2023},
      eprint={2310.01753},
      howpublished={arXiv preprint},
      primaryClass={cs.LG},
      url={https://doi.org/10.48550/arXiv.2310.01753}, 
}

@InProceedings{Causality4Climate,
  title = 	 {The Causality for Climate Competition},
  author =       {Runge, Jakob and Tibau, Xavier-Andoni and Bruhns, Matthias and Mu\~{n}oz-Mar\'{i}, Jordi and Camps-Valls, Gustau},
  booktitle = 	 {Proceedings of the NeurIPS 2019 Competition and Demonstration Track},
  pages = 	 {110--120},
  year = 	 {2020},
  editor = 	 {Escalante, Hugo Jair and Hadsell, Raia},
  volume = 	 {123},
  series = 	 {Proceedings of Machine Learning Research},
  month = 	 {08--14 Dec},
  publisher =    {PMLR},
  url = 	 {https://proceedings.mlr.press/v123/runge20a.html}
}

@article{equal_variance_1,
    author = {Peters, J. and Bühlmann, P.},
    title = "{Identifiability of Gaussian structural equation models with equal error variances}",
    journal = {Biometrika},
    volume = {101},
    number = {1},
    pages = {219-228},
    year = {2013},
    month = {11},
    issn = {0006-3444},
    doi = {10.1093/biomet/ast043},
    url = {https://doi.org/10.1093/biomet/ast043},
    eprint = {https://academic.oup.com/biomet/article-pdf/101/1/219/17460568/ast043.pdf},
}

@article{d-ball,
title = {On decompositional algorithms for uniform sampling from n-spheres and n-balls},
journal = {Journal of Multivariate Analysis},
volume = {101},
number = {10},
pages = {2297-2304},
year = {2010},
issn = {0047-259X},
doi = {10.1016/j.jmva.2010.06.002},
url = {https://www.sciencedirect.com/science/article/pii/S0047259X10001211},
author = {Radoslav Harman and Vladimír Lacko}
}

@article{River,
    author = {Tran, Ngoc Mai and Buck, Johannes and Klüppelberg, Claudia},
    title = "{Estimating a directed tree for extremes}",
    journal = {Journal of the Royal Statistical Society Series B: Statistical Methodology},
    pages = {qkad165},
    year = {2024},
    month = {02},
    issn = {1369-7412},
    doi = {10.1093/jrsssb/qkad165},
    url = {https://doi.org/10.1093/jrsssb/qkad165},
    eprint = {https://academic.oup.com/jrsssb/advance-article-pdf/doi/10.1093/jrsssb/qkad165/56666442/qkad165.pdf},
}

@misc{CausalChamber,
      title={The Causal Chambers: Real Physical Systems as a Testbed for AI Methodology}, 
      author={Juan L. Gamella and Jonas Peters and Peter Bühlmann},
      year={2024},
      eprint={2404.11341},
      howpublished={arXiv preprint},
      primaryClass={cs.AI},
      url={https://doi.org/10.48550/arXiv.2404.11341}
}

@misc{iSCM,
      title={Standardizing Structural Causal Models}, 
      author={Weronika Ormaniec and Scott Sussex and Lars Lorch and Bernhard Schölkopf and Andreas Krause},
      year={2024},
      eprint={2406.11601},
      howpublished={arXiv preprint},
      primaryClass={cs.LG},
      url={https://arxiv.org/abs/2406.11601}, 
}

@inproceedings{SAT,
  title={Constraint-based Causal Discovery: Conflict Resolution with Answer Set Programming.},
  author={Hyttinen, Antti and Eberhardt, Frederick and J{\"a}rvisalo, Matti},
  booktitle={UAI},
  pages={340--349},
  year={2014},
  url={https://dl.acm.org/doi/abs/10.5555/3020751.3020787}
}

@article{CAMPSVALLS20231,
    title = {Discovering causal relations and equations from data},
    journal = {Physics Reports},
    volume = {1044},
    pages = {1-68},
    year = {2023},
    issn = {0370-1573},
    doi = {10.1016/j.physrep.2023.10.005},
    url = {https://www.sciencedirect.com/science/article/pii/S0370157323003411},
    author = {Gustau Camps-Valls and Andreas Gerhardus and Urmi Ninad and Gherardo Varando and Georg Martius and Emili Balaguer-Ballester and Ricardo Vinuesa and Emiliano Diaz and Laure Zanna and Jakob Runge}
}

@InProceedings{Squires-DG,
  title = 	 {Causal Structure Discovery between Clusters of Nodes Induced by Latent Factors},
  author =       {Squires, Chandler and Yun, Annie and Nichani, Eshaan and Agrawal, Raj and Uhler, Caroline},
  booktitle = 	 {Proceedings of the First Conference on Causal Learning and Reasoning},
  pages = 	 {669--687},
  year = 	 {2022},
  editor = 	 {Schölkopf, Bernhard and Uhler, Caroline and Zhang, Kun},
  volume = 	 {177},
  series = 	 {Proceedings of Machine Learning Research},
  month = 	 {11--13 Apr},
  publisher =    {PMLR},
  url = 	 {https://proceedings.mlr.press/v177/squires22a.html}
}

@article{Mooij-DG,
  author  = {Joris M. Mooij and Sara Magliacane and Tom Claassen},
  title   = {Joint Causal Inference from Multiple Contexts},
  journal = {Journal of Machine Learning Research},
  year    = {2020},
  volume  = {21},
  number  = {99},
  pages   = {1--108},
  url     = {http://jmlr.org/papers/v21/17-123.html}
}

@article{assaad2022survey,
  title={Survey and evaluation of causal discovery methods for time series},
  author={Assaad, Charles K and Devijver, Emilie and Gaussier, Eric},
  journal={Journal of Artificial Intelligence Research},
  volume={73},
  pages={767--819},
  year={2022}, 
  url={https://doi.org/10.1613/jair.1.13428}
}

@article{Runge-Nature,
  title={Causal inference for time series},
  author={Runge, Jakob and Gerhardus, Andreas and Varando, Gherardo and Eyring, Veronika and Camps-Valls, Gustau},
  journal={Nature Reviews Earth \& Environment},
  volume={4},
  number={7},
  pages={487--505},
  year={2023},
  publisher={Nature Publishing Group UK London},
  doi={https://doi.org/10.1038/s43017-023-00431-y},
  url={https://www.nature.com/articles/s43017-023-00431-y}
}

@article{CD-Review,
author = {Nogueira, Ana Rita and Pugnana, Andrea and Ruggieri, Salvatore and Pedreschi, Dino and Gama, João},
title = {Methods and tools for causal discovery and causal inference},
journal = {WIREs Data Mining and Knowledge Discovery},
volume = {12},
number = {2},
pages = {e1449},
doi = {https://doi.org/10.1002/widm.1449},
url = {https://wires.onlinelibrary.wiley.com/doi/abs/10.1002/widm.1449},
eprint = {https://wires.onlinelibrary.wiley.com/doi/pdf/10.1002/widm.1449},
year = {2022}
}

@article{equal_variance_2,
    author = {Chen, Wenyu and Drton, Mathias and Wang, Y Samuel},
    title = "{On causal discovery with an equal-variance assumption}",
    journal = {Biometrika},
    volume = {106},
    number = {4},
    pages = {973-980},
    year = {2019},
    month = {09},
    issn = {0006-3444},
    doi = {10.1093/biomet/asz049},
    url = {https://doi.org/10.1093/biomet/asz049},
    eprint = {https://academic.oup.com/biomet/article-pdf/106/4/973/30646770/asz049.pdf},
}

@book{CoD,
author={Stanton A. Glantz and Bryan K. Slinker and Torsten B. Neilands},
title={Primer of Applied Regression and Analysis of Variance},
edition = "3",
publisher={McGraw-Hill Education},
address={New York, NY},
year={2017},
url={accessbiomedicalscience.mhmedical.com/content.aspx?aid=1141896746}
}

@article{ER,
  title={On the evolution of random graphs},
  author={Erdös, Paul and R{\'e}nyi, Alfr{\'e}d},
  journal={Publ. Math. Inst. Hungar. Acad. Sci},
  volume={5},
  pages={17--61},
  year={1960},
  publisher={Citeseer}
}

@article{
SF,
author = {Albert-László Barabási  and Réka Albert },
title = {Emergence of Scaling in Random Networks},
journal = {Science},
volume = {286},
number = {5439},
pages = {509-512},
year = {1999},
doi = {10.1126/science.286.5439.509},
URL = {https://www.science.org/doi/abs/10.1126/science.286.5439.509},
eprint = {https://www.science.org/doi/pdf/10.1126/science.286.5439.509}}

@misc{DaO,
      title={Better Simulations for Validating Causal Discovery with the DAG-Adaptation of the Onion Method}, 
      author={Bryan Andrews and Erich Kummerfeld},
      year={2024},
      howpublished={arXiv preprint},
      url = {https://arxiv.org/abs/2405.13100}
}

@article{hardness,
author = {Tor\'{a}n, Jacobo},
title = {On the Hardness of Graph Isomorphism},
journal = {SIAM Journal on Computing},
volume = {33},
number = {5},
pages = {1093-1108},
year = {2004},
doi = {10.1137/S009753970241096X},

URL = { https://doi.org/10.1137/S009753970241096X},
eprint = {https://doi.org/10.1137/S009753970241096X}
}

@inbook{Pearl, place={Cambridge}, title={A Theory of Inferred Causation}, booktitle={Causality}, publisher={Cambridge University Press}, author={Pearl, Judea}, year={2009}, pages={41–64}}

@inbook{SVAR,
  title={New introduction to multiple time series analysis},
  author={L{\"u}tkepohl, Helmut},
  year={2005},
  publisher={Springer Science \& Business Media},
  chapter={9.1.1 The $\mathsf{A}$-Model}
}

@inbook{stability,
  title={New introduction to multiple time series analysis},
  author={L{\"u}tkepohl, Helmut},
  year={2005},
  publisher={Springer Science \& Business Media},
  chapter={2.1.1 Stable VAR($p$) Processes}
}

@article{OU,
  title={Multivariate autoregressive and Ornstein-Uhlenbeck processes: estimates for order, parameters, spectral information, and confidence regions},
  author={Neumaier, Arnold and Schneider, Tapio},
  journal={ACM Transactions in Mathematical Software},
  year={1998},
  publisher={Citeseer}
}

@misc{PC,
      title={Estimating high-dimensional directed acyclic graphs with the PC-algorithm}, 
      author={Markus Kalisch and Peter Buehlmann},
      year={2005},
      eprint={math/0510436},
      howpublished={arXiv preprint},
      primaryClass={math.ST},
      url={https://arxiv.org/abs/math/0510436}, 
}

@misc{gcastle,
  title={gCastle: A Python Toolbox for Causal Discovery}, 
  author={Keli Zhang and Shengyu Zhu and Marcus Kalander and Ignavier Ng and Junjian Ye and Zhitang Chen and Lujia Pan},
  year={2021},
  eprint={2111.15155},
  howpublished={arXiv preprint},
  primaryClass={cs.LG},
  url={https://doi.org/10.48550/arXiv.2111.15155}
}

@article{GES,
  title={Optimal structure identification with greedy search},
  author={Chickering, David Maxwell},
  journal={The Journal of Machine Learning Research},
  volume={3},
  pages={507--554},
  year={2003},
  publisher={JMLR. org},
  url={https://doi.org/10.1162/153244303321897717}
}

@book{peters2017elements,
  title={Elements of causal inference: foundations and learning algorithms},
  author={Peters, Jonas and Janzing, Dominik and Sch{\"o}lkopf, Bernhard},
  year={2017},
  publisher={The MIT Press}
}

@Book{Spirtes2000,
  title     = {{Causation, Prediction, and Search}},
  publisher = {MIT Press},
  year      = {2000},
  author    = {Spirtes, Peter and Glymour, Clark and Scheines, Richard},
  address   = {Boston},
}

@article{runge2019inferring,
  title={Inferring causation from time series in Earth system sciences},
  author={Runge, Jakob and Bathiany, Sebastian and Bollt, Erik and Camps-Valls, Gustau and Coumou, Dim and Deyle, Ethan and Glymour, Clark and Kretschmer, Marlene and Mahecha, Miguel D and Mu{\~n}oz-Mar{\'\i}, Jordi and others},
  journal={Nature communications},
  volume={10},
  number={1},
  pages={1--13},
  year={2019},
  publisher={Nature Publishing Group},
  url={https://www.nature.com/articles/s41467-019-10105-3},
  doi={https://doi.org/10.1038/s41467-019-10105-3}
}

@article{tsg,
  title = {Quantifying causal coupling strength: A lag-specific measure for multivariate time series related to transfer entropy},
  author = {Runge, Jakob and Heitzig, Jobst and Marwan, Norbert and Kurths, J\"urgen},
  journal = {Phys. Rev. E},
  volume = {86},
  issue = {6},
  pages = {061121},
  numpages = {15},
  year = {2012},
  month = {Dec},
  publisher = {American Physical Society},
  doi = {10.1103/PhysRevE.86.061121},
  url = {https://link.aps.org/doi/10.1103/PhysRevE.86.061121}
}

@misc{TSgeneration,
      title={Data Generating Process to Evaluate Causal Discovery Techniques for Time Series Data}, 
      author={Andrew R. Lawrence and Marcus Kaiser and Rui Sampaio and Maksim Sipos},
      year={2021},
      eprint={2104.08043},
      howpublished={arXiv preprint},
      primaryClass={stat.ML},
      url={https://doi.org/10.48550/arXiv.2104.08043}, 
}

@misc{Nicolas,
  title={Causal Inference on Process Graphs, Part I: The Structural Equation Process Representation},
  author={Reiter, Nicolas-Domenic and Gerhardus, Andreas and Wahl, Jonas and Runge, Jakob},
  howpublished={arXiv preprint},
  url = {https://doi.org/10.48550/arXiv.2305.11561},
  year={2023}
}

@article{runge2018causal,
  title={Causal network reconstruction from time series: From theoretical assumptions to practical estimation},
  author={Runge, Jakob},
  journal={Chaos: An Interdisciplinary Journal of Nonlinear Science},
  volume={28},
  number={7},
  pages={075310},
  year={2018},
  publisher={AIP Publishing LLC},
  url={https://pubs.aip.org/aip/cha/article/28/7/075310/386353}
}

@article{kaeding2023distinguishing,
  author  = {K\"ading, Christoph and Runge, Jakob},
  title   = {Distinguishing Cause and Effect in Bivariate Structural Causal Models: A Systematic Investigation},
  journal = {Journal of Machine Learning Research},
  year    = {2023},
  volume  = {24},
  number  = {278},
  pages   = {1--144},
  url     = {http://jmlr.org/papers/v24/22-0151.html}
}

@article{mooij2016distinguishing,
  title={Distinguishing cause from effect using observational data: methods and benchmarks},
  author={Mooij, Joris M and Peters, Jonas and Janzing, Dominik and Zscheischler, Jakob and Sch{\"o}lkopf, Bernhard},
  journal={Journal of Machine Learning Research},
  volume={17},
  number={32},
  pages={1--102},
  year={2016},
  url={https://jmlr.org/papers/volume17/14-518/14-518.pdf}
}

@book{guyon2019cause,
  title={Cause effect pairs in machine learning},
  author={Guyon, Isabelle and Statnikov, Alexander and Batu, Berna Bakir},
  year={2019},
  publisher={Springer}
}

@misc{brouillard2024landscape,
  title={The Landscape of Causal Discovery Data: Grounding Causal Discovery in Real-World Applications},
  author={Brouillard, Philippe and Squires, Chandler and Wahl, Jonas and Kording, Konrad P and Sachs, Karen and Drouin, Alexandre and Sridhar, Dhanya},
  howpublished={arXiv preprint},
  year={2024},
  url={https://doi.org/10.48550/arXiv.2412.01953}
}

\appendix
\renewcommand\thefigure{A\arabic{figure}}
\renewcommand\thealgorithm{A\arabic{algorithm}}
\setcounter{figure}{0} 
\setcounter{algorithm}{0} 

\vspace{-6pt}
\section{The Scaling Product}\label{Appendix Scaling}
\begin{dfn}
    $\vec \Theta$ is a \emph{\textbf{parameterization}} of some stochastic functional space $\mathcal F$ over variables $V$ if any stochastic structural assignment $X_i:=f(V)$ with $f\in \mathcal F$ can be notated by a vector of parameters $\vec \Theta_i$.
\end{dfn}

\noindent The space of linear additive Gaussian noise functions defined by Equation \ref{eq:SCM} can be parameterized by a vector consisting of causal coefficients followed by the noise mean and standard deviation. For a node $X_i$ in a graph with $N$ nodes, $\vec \Theta_i = (a_{1i},\dots,a_{Ni}, m_i, s_i)$. To extract the parameters that are random variables, we have the shorter vector $\vec{\tilde \Theta}_i = (a_{ji} | X_j\in pa(X_i),m_i,s_i)$.

\begin{dfn}\label{SP}
    The \emph{\textbf{scaling product}} $\odot$ associated with a parameterization $\vec \Theta$ of a functional space $\mathcal F$ is defined such that $f_i\in \mathcal F \Rightarrow cf_i \in \mathcal F$ can be expressed $c\odot \vec \Theta_i \; \forall c\in \mathbb R$. 
\end{dfn}

\noindent The scaling product for the linear additive Gaussian noise functional parameterization discussed above is simply element-wise multiplication because $f_i(V) = \sum_{j=1}^{\#V} a_{ji}X_j + U_i, \; U_i\sim\mathcal N(m_i,s_i) \Rightarrow cf_i = c(\sum_{j=1}^{\#V} a_{ji}X_j + U_i) = \sum_{j=1}^{\#V} (ca_{ji})X_j + U'_i, \; U'_i\sim\mathcal N(cm_i,cs_i)$ if $\mathcal N(m_i,s_i)$ is defined as the normal distribution centered at $m_i$ with standard deviation $s_i$. However, if $\mathcal N(m_i, s_i)$ is defined as the normal distribution with \emph{variance} $s_i$, then $\odot$ is defined such that $c \odot \vec \Theta = (ca_{1j},\dots,ca_{Nj},cm_i,\sqrt{c}s_i)$. Any parameterization will induce such a scaling product.

\newpage
\section{Modified Sortability Definitions}\label{alg:sort}

Here, we present code for calculating sortability which is modified from \cite{Beware} [\url{https://github.com/CausalDisco/CausalDisco/blob/main/LICENSE}] such that it (1) accepts cyclic graphs and avoids comparing nodes within the same cyclic component (lines 13-17, 29, and 38-39), and (2) compares each pair of nodes connected by a causal path exactly once (lines 24, 28, and 35 are added, and \texttt{check\_now} replaces \texttt{Ek} in lines 30-34). Then, in Figure \ref{fig:R2ex}, we demonstrate the difference between the time series extensions of R2-sortability proposed by \cite{Christopher}, denoted R2*-sortability, and that proposed here, denote R2-sortability. 

\begin{python}
def sortability(M, W, tol=1e-9):
    '''Takes a 1 x d metric M, such as variance or R2, and an N x N adjacency 
    matrix W, where the j,i-th entry corresponds to the edge weight for j->i,
    and returns a value indicating how well M reflects the causal order.
    Modified from <https://github.com/Scriddie/Varsortability> to avoid double-
    counting node pairs connected by causal paths of multiple lengths, and to 
    accept graphs with cycles in the manner described by Christopher Lohse and 
    Jonas Wahl in "Sortability of Time Series Data" (Submitted to the Causal 
    Inference for Time Series Data Workshop at the 40th Conference on 
    Uncertainty in Artificial Intelligence). '''
    E = W != 0
    #Find ancestral relationships to avoid comparison within cycles
    Ek = E.copy()
    anc = Ek.copy()*False
    for path_len in range(self.N):
        anc = anc | Ek
        Ek = Ek.dot(E)
    #reset Ek to keep track of paths of various lengths
    Ek = E.copy()

    n_paths = 0
    n_correctly_ordered_paths = 0
    
    checked_paths = Ek.copy()*False

    for path_len in range(E.shape[0] - 1):
        check_now = (Ek 
                     & ~ checked_paths # to avoid double counting
                     & ~ anc.T) #to avoid comparison within a cycle
        n_paths += check_now.sum()
        n_correctly_ordered_paths += (check_now * M / M.T > 1 + tol).sum()
        n_correctly_ordered_paths += 1/2*(
            (check_now * M / M.T <= 1 + tol) *
            (check_now * M / M.T >=  1 - tol)).sum()
        checked_paths = checked_paths | check_now
        Ek = Ek.dot(E) #examine paths of length path_len+=1

    if n_paths == 0: #in case all nodes are in the same cycle
        return 0.5
    return n_correctly_ordered_paths / n_paths
\end{python}

\begin{figure}[H]
    \centering
    \includegraphics[width=0.9\linewidth]{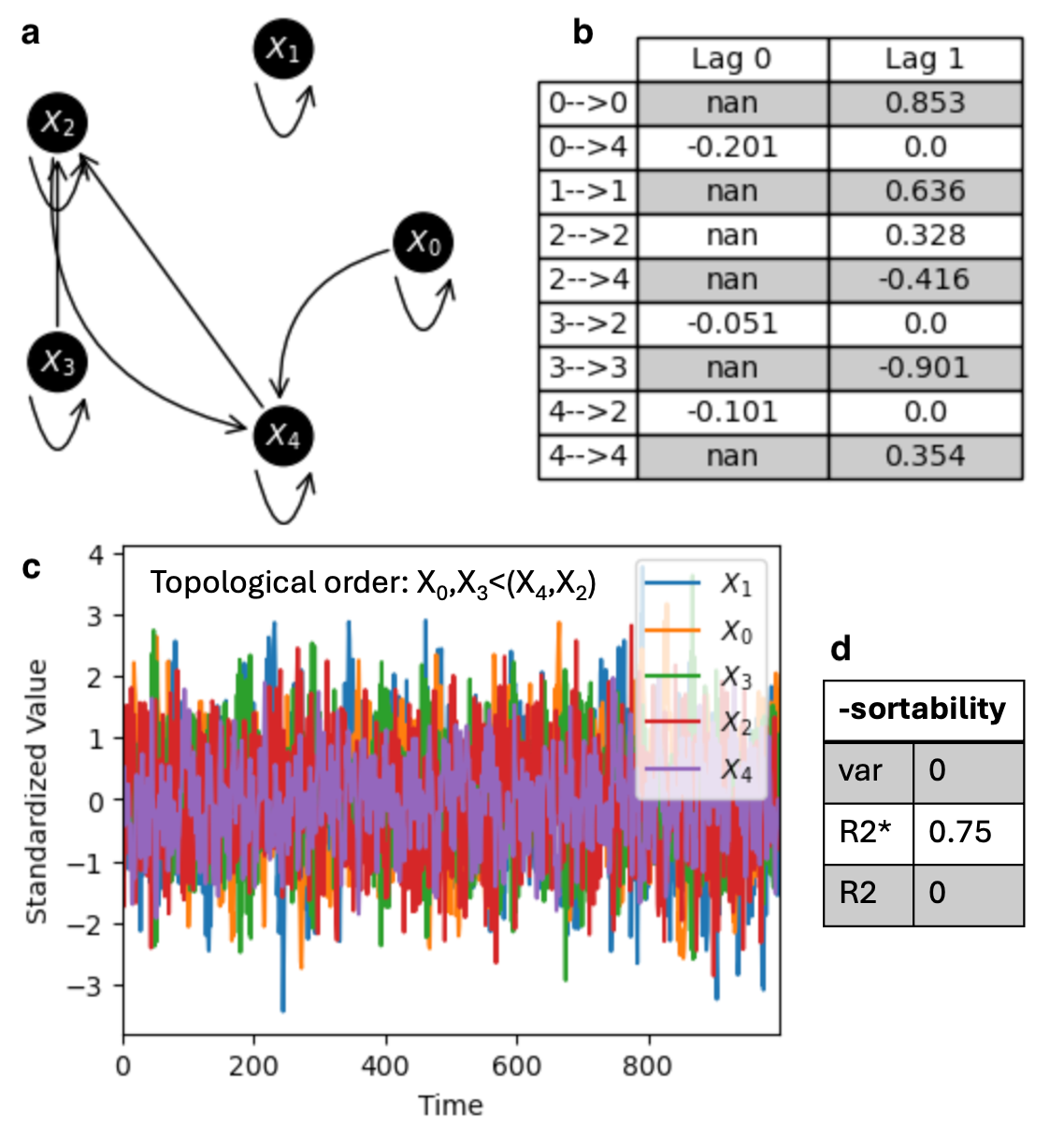}
    \caption{\textbf{An example SVAR model where reverse-varsortability is associated with R2*-sortability, but with reverse-R2-sortability.} \textbf{a} The summary graph for the underlying SVAR model with $N=5$ and $\tau_{\max}=1$. \textbf{b} The coefficients from the underlying SVAR model that are associated with each edge from the summary graph, where different lags appear in different columns. In this table, $0.0$ and \texttt{nan} appear when there is an edge that exists at one lag but not the other. \texttt{nan} is used instead of $0.0$ if a 0-lag edge in that orientation would run counter to the topological order of the variables in the SVAR model. \textbf{c} Simulated data for this SVAR process, where the time series are ordered by decreasing variance for best visibility. The topological order of the variables according to the summary graph is written at the top of this plot. \textbf{d} Time series var-, R2*-, and R2-sortability for this dataset.} 
    \label{fig:R2ex}
\end{figure}

\section{Benchmark Sortability Properties}
Figure \ref{fig:data_gen} compares the sortability properties of five data generation methods discussed in Sections \ref{Existing} and \ref{remedies}. 
The standard unit-variance-noise approach \citep[first row]{NOTEARS} produces highly var- and R2-sortable datasets. The IPA approach \cite [second row]{Mooij-DG} reduces and reverses the sortability tendencies of the data, while the 50-50 approach \cite[third row]{Squires-DG} practically eliminates trends in varsortability but produces strongly reverse-R2-sortable data. The iSCM \citep[fourth row]{iSCM} eliminates varsortability and produces data that is gently, but noticeably, reverse-R2-sortable. Finally, DaO \citep[last row]{DaO} produces data that is reverse varsortable and whose R2-sortability tendencies change with density. All trends except R2-sortability for DaO become stronger for denser graphs (last column).
\begin{figure}[H]
    \centering
    \includegraphics[scale=0.8]{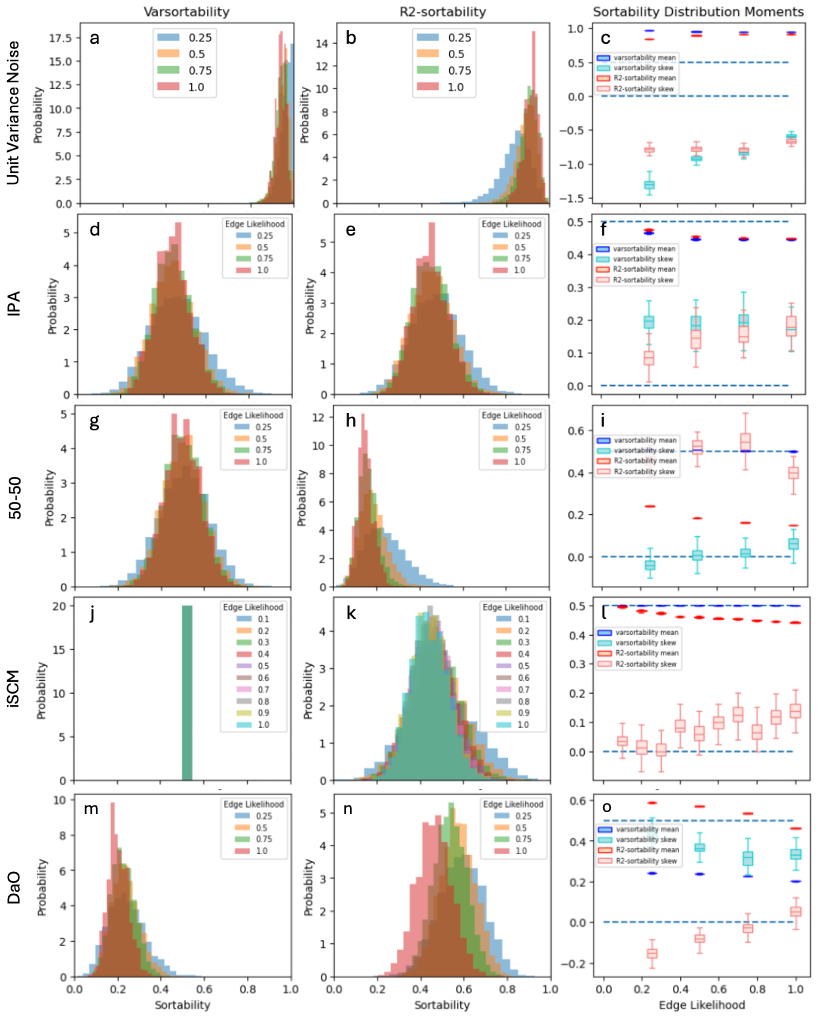}
    \caption{Var- and R2-sortability propbability distribution functions (pdfs, left and middle columns, respectively) for static ER SCMs with $N=20$ and varying densities (colors) generated via different methods described in Sections \ref{Existing} and \ref{remedies} (different rows). The means and skews of the var- and R2-sortability pdfs are displayed in the right column as a function of graph density in blue and red, respectively, where the dashed horizontal lines at 0.5 and 0 mark neutral mean and skew, respectively. Pdfs are based on 5000 randomly generated SCMs with 100 observations.}
    \label{fig:data_gen}
\end{figure}

\newpage
\section{R2 Asymmetry} \label{Appendix R2}
\subsection{R2-score}
The R2-score for a variable $\hat{X}_i$ in a standardized dataset $\boldsymbol{X}\in \mathbb R^{M \times N}$ with variables $\{\hat{X}_i|i\in 1...N\}$ and $M$ samples is \[R^2(\hat{X}_i) = 1-\text{var} (\hat{X}_i-E[\hat{X}_i|\{\hat{X}_j|j\neq i\}])\]
We can estimate $E[\hat{X}_i|\{\hat{X}_j|j\neq i\}]$ using linear regression. The solution to the regression \newline \(\mathbf y=\mathbf X\beta + \mathbf \epsilon\) is \[\beta=(\mathbf X^T\mathbf X)^{-1}\mathbf X^T\mathbf y\]
yielding residuals
\begin{align*}
    \epsilon &= \mathbf y - \mathbf X\beta = \mathbf y - \mathbf X(\mathbf X^T\mathbf X)^{-1}\mathbf X^T\mathbf y = (\mathbf I_M - \mathbf X(\mathbf X^T\mathbf X)^{-1}\mathbf X^T)\mathbf y
\end{align*}
that estimate $\hat{X}_i-E[\hat{X}_i|\{\hat{X}_j|j\neq i\}]$. 

\subsection{SCMs for Unshielded Triples}
We examine an SCM with standardized variables \(\{\hat{A}, \hat{B}, \hat{C}\}\) and noise terms \(\hat{U} = [\hat{U}_A, \hat{U}_B, \hat{U}_C]\), where $\mathbf{U} = [\mathbf{U_A}, \mathbf{U_B}, \mathbf{U_C}] \in \mathbb R^{M \times 3}$ is a matrix consisting of \(M\) samples of noise terms drawn iid from \(\mathcal N(0,1)\), and $[\mathbf A, \mathbf B, \mathbf C]\in \mathbb R^{M \times 3}$ is the resulting dataset. 

\begin{multicols}{2}
\subsubsection{Collider}
If \(\hat{A}\rightarrow \hat{B} \leftarrow \hat{C}\) is a collider, it can be represented with the SCM:
\[\Biggl\{ \begin{array}{ll} \hat{A} := \hat U_A \\ \hat C := \hat U_C \\ \hat B := a\hat A + c\hat C + b\hat U_B\end{array}\]
where \(a^2 + c^2 + b^2 = 1\). The sample values for an observational dataset can be expressed: 
\[\begin{bmatrix} \mathbf A & \mathbf B & \mathbf C \end{bmatrix} := \mathbf U \begin{bmatrix} 1 & a & 0\\0&b & 0\\0 & c & 1\end{bmatrix}\]
and the covariance matrix for the dataset is
\[\begin{bmatrix}
    1 & a & 0 \\ a & 1 & c \\ 0 & c & 1
\end{bmatrix}\]

\subsubsection{Confounder}
If \(\hat{A} \leftarrow \hat B \rightarrow \hat C\) is a confounded triple, it can be represented with the SCM: 
\[\Biggl\{ \begin{array}{ll} \hat B := \hat U_B \\ \hat A := a\hat B + \bar{a}\hat U_A \\ \hat C := c\hat B + \bar{c}\hat U_C\end{array}\]
where \(\bar{x} = \sqrt{1-x^2}\). 
The sample values for an observational dataset can be expressed: 
\[\begin{bmatrix} \mathbf A & \mathbf B & \mathbf C \end{bmatrix} := \mathbf U \begin{bmatrix} \bar{a} & 0 & 0\\ a& 1 & c\\0 & 0 & \bar{c}\end{bmatrix}\]
and the covariance matrix for the dataset is
\[\begin{bmatrix}
    1 & a & ac \\ a & 1 & c \\ ac & c & 1
\end{bmatrix}\]
\end{multicols}

\subsection{R2 Properties of Hub Nodes in Unshielded Triples}
If we regress the middle node \(\mathbf B\) on \(\mathbf A\) and \(\mathbf C\) in the collider, where our regression is consistent with the causal order, the regression will recover the causal model in the infinite sample limit. Thus, we will find \(\epsilon \geq b\mathbf{U_B}\), yielding $R^2_B \leq 1-b^2$ that is upper bounded by the cause-explained variance.

If we perform the same regression on the confounder instead, then substituting $\mathbf X = [\mathbf A, \mathbf C]$ and $\mathbf y = \mathbf B$ yields 

\begin{align*}
    \mathbf X^T\mathbf X &= \begin{bmatrix} \mathbf A^T \\ \mathbf C^T \end{bmatrix}\begin{bmatrix} \mathbf A & \mathbf C \end{bmatrix} = M\begin{bmatrix}
        1 & ac \\ ac & 1
    \end{bmatrix} \\
    (\mathbf X^T\mathbf X)^{-1} &= \frac{1}{M}\frac{1}{1-a^2c^2} \begin{bmatrix} 1 & -ac \\ -ac & 1 \end{bmatrix} \\
    \mathbf X(\mathbf X^T\mathbf X)^{-1}\mathbf X^T &= \begin{bmatrix} \mathbf A & \mathbf C\end{bmatrix} \frac{1}{M}\frac{1}{1-a^2c^2} \begin{bmatrix} 1 & -ac \\ -ac & 1 \end{bmatrix} \begin{bmatrix} \mathbf A^T \\ \mathbf C^T \end{bmatrix} \\
    &= \frac{1}{M}\frac{\mathbf A\mathbf A^T+\mathbf C\mathbf C^T - ac(\mathbf A\mathbf C^T+\mathbf C\mathbf A^T)}{1-a^2c^2} \\
    \Rightarrow \epsilon &= \left( \sum_j \left(\delta_{ij} - \frac{1}{M}\frac{a_ia_j + c_ic_j-ac(a_ic_j+c_ia_j)}{(1-a^2c^2)}\right)b_j\right)_i \\
    &= \mathbf B - \frac{\sigma_{AB}\mathbf A+\sigma_{BC}\mathbf C-ac(\sigma_{BC}\mathbf A+\sigma_{AB}\mathbf C)}{1-a^2c^2} \\
    &= \mathbf B - \frac{a\mathbf A+c\mathbf C-ac(c\mathbf A+a\mathbf C)}{1-a^2c^2} = \mathbf B - \frac{a(1-c^2)\mathbf A+c(1-a^2)\mathbf C}{1-a^2c^2} \\
    &= \mathbf B - \frac{a\bar{c}^2\mathbf A+c\bar a^2\mathbf C}{1-a^2c^2} = \mathbf B - \frac{a\bar{c}^2 (a\mathbf B + \bar a \mathbf U_A)+c\bar a^2(c\mathbf B + \bar c \mathbf U_C)}{1-a^2c^2} \\
    &= \frac{1-a^2c^2-a^2\bar c^2 - \bar a^2c^2}{1-a^2c^2}\mathbf B - \frac{a\bar a \bar c^2}{1-a^2c^2}\mathbf U_A - \frac{c\bar c \bar a^2}{1-a^2c^2}\mathbf U_C \\
    &= \frac{\bar a^2\bar c^2\mathbf B -a\bar a \bar c^2\mathbf U_A- c\bar c \bar a^2\mathbf U_C }{1-a^2c^2} \\
    \Rightarrow \sigma_\epsilon^2 &= \frac{\bar a^4 \bar c^4 +a^2\bar a^2\bar c^4 + c^2\bar c^2\bar a^4}{(1-a^2c^2)^2}  = 
    \frac{\bar a^4 \bar c^4 +(1-\bar a^2)\bar a^2\bar c^4 + (1-\bar c^2)\bar c^2\bar a^4}{(1-(1-\bar a^2)(1-\bar c^2))^2} \\
    &= \frac{\bar a^2\bar c^2(\bar a^2 + \bar c^2 - \bar a^2\bar c^2)}{(\bar a^2 + \bar c^2 - \bar a^2\bar c^2)^2} = \frac{\bar a^2 \bar c^2}{\bar a^2 + \bar c^2 - \bar a^2\bar c^2}\end{align*}
Without loss of generality, let \(\bar{a}^2=s^2 < 1\) be the larger noise term and \(\bar{c}^2=t^2s^2 < 1\) be the smaller noise term, with \(t^2<1\). Then
\[\sigma_\epsilon^2 = \frac{s^2(s^2t^2)}{s^2+s^2t^2-s^2(s^2t^2)} = \frac{s^2t^2}{1 +t^2(1- s^2)}<s^2t^2 < s^2\]
and $R^2_B = 1-\sigma_\epsilon^2 > 1-s^2t^2$ is lower-bounded by 1 minus the smaller of the two noise terms.

\section{Time Series}\label{Time Series Appendix}
Our formulation treats time discretely, as is common to several time series causal inference works, but many relevant systems we may wish to emulate exist in continuous time. Thus, we must understand the relationship between parameters from continuous differential equations and the causal coefficients in our discrete-time systems, and we begin by examining the continuous analog of autoregressive models: the Ornstein-Uhlenbeck process.
\subsection{Continuous Analog: Ornstein-Uhlenbeck Processes}\label{OU}
A univariate zero-mean Ornstein-Uhlenbeck process $X(t)$ is defined by the \textit{Langevin equation} \[dX = -\kappa X dt + \sigma dW\] where $W$ denotes a Wiener process with the property $W(t+\tau)-W(t) \sim \mathcal N(0,\tau)$ \citep{OU}. In this formulation, $\kappa \in \mathbb R^+$ is the rate of reversion to the mean and $\sigma \in \mathbb R^+$ is the noisiness of the process, and the variance of this system is $\sfrac{\sigma^2}{2\kappa}$ (note $\kappa\in\mathbb R^+$). As with our discrete formulations, $\kappa$ is naturally unitless and would not be affected by standardizing $X$ -- instead, $\sigma$ would be set to $\sqrt{2\kappa}$. A discrete sampling with time-step $\Delta$ after standardizing is 
\[\hat{X}_\Delta (t+1) = e^{-\kappa \Delta}\hat{X}_\Delta(t) + \hat{U}(t),\; \hat{U}(t)\sim \mathcal N(0,\sqrt{1-e^{-2\kappa\Delta}})\] with causal coefficient $\hat{a}(\Delta) = (e^{-\kappa})^\Delta$ and noise standard deviation $\hat{s}(\Delta) = \sqrt{1-\hat{a}(\Delta)^2}$. This univariate system has one degree of freedom, parameterized by $\Delta$ -- the inverse of the sampling rate. 

A multivariate Ornstein-Uhlenbeck process has more than one degree of freedom even when it has zero-mean and no interaction terms ($j\neq i \Rightarrow \kappa_{ji}=0$), but there is still only one sampling rate. Thus, in addition to drawing a parameter related to $\Delta$, we must draw parameters related to $\kappa$. As in the static case, it is acceptable if we draw more parameters than degrees of freedom as long as we obey the constraints of the unitless system. The parameters $\kappa_{ii}$ must be drawn from $\mathbb R^+$, but $e^{-\kappa_{ii}}\in(0,1)$ falls in the unit interval. $\Delta$ should also be drawn from $\mathbb R^+$, but a scientist employing time series causal methods should attempt to choose a sampling frequency that produces non-trivial auto-coefficients and noise terms. Thus, we choose to draw $e^{-\kappa} \sim \mathcal U(0,1)$ uniformly and $\Delta \sim \mathcal F(100,100)$ from Snedecor's $F$ distribution with $d_1=d_2=100$, which spans $\mathbb R^+$ and centers at $1$. Since these auto-dependence parameters will not be affected by standardization, the cross-dependence terms and the standard deviation of the noise must adjust to fit the drawn auto-dependence. Given contemporaneous cross-dependencies, we can divide the remaining variance between the cross-dependence terms and noise using a $d_i$-ball in a manner similar to the static case. 

How should we handle lagged cross-dependencies? Examine a bivariate Ornstein-Uhlenbeck process with a one-way causal dependency: 
\begin{align*}
    d\begin{bmatrix}
        X_1(t) \\ X_2(t)
    \end{bmatrix} = -\begin{bmatrix}
        \kappa_{11} & 0 \\ \kappa_{12} & \kappa_{22}
    \end{bmatrix}\begin{bmatrix}
        X_1(t) \\ X_2(t)
    \end{bmatrix}dt + \begin{bmatrix}
        \sigma_1 & 0 \\ 0 & \sigma_2
    \end{bmatrix}d\begin{bmatrix}
        W_1(t) \\ W_2(t)
    \end{bmatrix}
\end{align*}
In its standardized form, $\hat{\kappa}_{ii}=\kappa_{ii}$, $\hat{\kappa}_{12}=\frac{\sigma_1}{\sigma_2}\sqrt{\frac{\kappa_{22}}{\kappa_{11}}}\kappa_{12}$, $ \hat{\sigma}_1^2= 2\kappa_{11}$, and $\hat{\sigma}_2^2 = 2\kappa_{22}\left(1-\frac{2\hat{\kappa}_{12}^2}{\kappa_{22}(\kappa_{11}+\kappa_{22})}\right)$. While the auto-dependence terms $\kappa_{ii}\in\mathbb R^+$ must always be positive, the cross-dependence term $\hat{\kappa}_{12}\in\mathbb R$ may be negative. A discrete sampling of this standardized process with time-step $\Delta$ is
\begin{align*}
    \begin{bmatrix}
        \hat{X}_1(t+1) \\ \hat{X}_2(t+1)
    \end{bmatrix}_\Delta 
    = \begin{bmatrix}
        e^{-\kappa_{11}\Delta} & 0 \\ \hat{\kappa}_{12}\frac{e^{-\kappa_{11}\Delta} - e^{-\kappa_{22}\Delta}}{\kappa_{11}-\kappa_{22}} & e^{-\kappa_{22}\Delta}
    \end{bmatrix}\begin{bmatrix}
        \hat{X}_1(t) \\ \hat{X}_2(t)
    \end{bmatrix}_\Delta + \begin{bmatrix}
        \hat{U}_1(t) \\ \hat{U}_2(t)
    \end{bmatrix}_\Delta
\end{align*} 
The cross-term is a function of the auto-dependence terms $\kappa_{ii}$ and $\Delta$: $\hat{a}_{12}(\Delta) = \frac{\hat{\kappa}_{12}}{\kappa_{11}-\kappa_{22}}(\hat{a}_{11}(\Delta)-\hat{a}_{22}(\Delta))$. Since we draw $e^{-\kappa}$ from a uniform distribution and $\Delta$ from a distribution centered at 1, rather than drawing $\kappa$ and $\Delta$ from $\mathbb R^+$, we will draw $\frac{\hat{\kappa}_{12}}{\kappa_{11}-\kappa_{22}}$ in a manner similar to the static case: starting with an initial draw from $\mathcal N(0,1)$, dividing by the norm from all parents (in this case, canceling out that initial draw), and multiplying by $r_2\sim \mathcal U(0,1)$ (in this case, there is only one parent). Though we do not consider the auto-dependence to be a parent when making the initial draw for the cross-dependence, we must consider it when scaling to create a normalized time series. Since the auto-covariance depends on the cross-dependence, we must solve a system of equations.

With $v_\Delta(x) = \frac{1- e^{-x\Delta}}{x}$, the variances of the noise terms $\vec{\hat{U}}_\Delta(t)$ are:
\begin{align*}
    \hat{s}_1^2 &= \hat{\sigma}_1^2v_\Delta(2\kappa_{11}) = 2\kappa_{11}v_\Delta(2\kappa_{11}) = 1 - \hat{a}_{11}^2(\Delta) 
    \\ 
    \hat{s}_2^2 
    &= \big[\hat{\sigma}_1^2\hat{\kappa}_{12}^2\big(v_\Delta(2\kappa_{11}) - 2v_\Delta(\kappa_{11}+\kappa_{22}) + v_\Delta(2\kappa_{22})\big) \\
    &+ \hat{\sigma}_2^2(\kappa_{11}-\kappa_{22})^2v_\Delta(2\kappa_{22})\big]/(\kappa_{11}+\kappa_{22})^2
\end{align*}

The expression for $\hat{s}_1$ is fairly simple, echoing the expression from the static case. The expression for $\hat{s}_2$ is much more complicated because the total variance of $\hat{X}_2$ used for standardization depends on the system of equations defining the lagged co-variance terms. We will solve for this, rather than calculating it outright. There cross-covariance term is \[    \hat{s}_{12} = \frac{\hat{\sigma}_1^2\hat{\kappa}_{12}}{\kappa_{11}-\kappa_{22}}\left(v_\Delta(2\kappa_{11}) - v_\Delta(\kappa_{11}+\kappa_{22})\right) = \frac{2\kappa_{11}\hat{\kappa}_{12}}{\kappa_{11}-\kappa_{22}}\left(v_\Delta(2\kappa_{11}) - v_\Delta(\kappa_{11}+\kappa_{22})\right)\]
It is non-trivial whenever $\Delta\neq 0$, and positive when $\kappa_{12}<0$. This implies that discrete sampling of a continuous system will \textit{always} produce hidden confounding, but the magnitude of the confounding can be controlled by adjusting $\Delta$. For simplicity, let's examine the situation when $\kappa_{22}=\kappa_{11}=\kappa$. The magnitude of the confounding $\frac{|\kappa_{12}|}{2\kappa}(1-(2\kappa\Delta + 1)e^{-2\kappa\Delta}) < p$ is less than some value $p>0$ as long as $\Delta < -\frac{1}{2}\left(W_{-1}\left(-\frac{1}{e}\left(1-\frac{2\kappa}{|\kappa_{12}|}p\right)\right)+1\right)$, where $W_{-1}$ is the $-1$ branch of the Lambert W-Function. Thus, if $\Delta$ is chosen to be small enough, hidden confounding could be reasonably controlled. However, one must balance this goal with keeping $\Delta$ large enough that each variable has non-negligible noise at each time step. 

\subsection{Multiple Dependencies}\label{mult}
If the discretization of a continuous process is an AR($p$) process with $p>1$, then the continuous process is described by a differential equation of order $>1$. The number of auto-dependencies at discrete lags reflects the degree of non-linearity of the autodependence, rather than distinct parents that reduce the expected fraction of variance due to noise, as in the static case. A non-linear effect cannot be represented with a single parameter, but its effect on the mean of the child process can be quantified using the direct transfer function at frequency 0 \citep{Nicolas}:
\[\hat{\mathfrak{h}}_{ji}(0) = \frac{\sum_{\tau=0}^{\tau_{\max}} \hat{a}_{ji}(\tau)}{1 - \sum_{\tau=0}^{\tau_{\max}} \hat{a}_{ii}(\tau)}\]
This measure includes the indirect effect of $X_j$ on $X_i$ via the past of $X_i$ ($\hat{a}_{ii}(\tau)$), but, as discussed in Sections \ref{TS Exp Intro} and \ref{OU}, we wish to treat auto-dependence separately. Thus, we represent the contribution of each parent using only the numerator of this expression: $\sum_\tau \hat{a}_{ji}(\tau)$. These oscillatory processes have associated characteristic frequencies, and the coefficients in the discretization may vary widely based on the ratio of the sampling frequency to the characteristic frequency in a process called \textit{aliasing}. 

Sampling directly from frequency space is out of the scope of this paper, and since we are doing random generation, it does not matter if we preserve characteristic frequencies between our first draw and our final parameters. Therefore, we will do the initial draws of $a_{ji}^{\prime\prime}(\tau)\sim \mathcal N(0,1)$, but before moving on, we'll also draw contribution parameters $b_{ji}'$ from $\mathcal U(0,1)$ if $i=j$ and from $\mathcal N(0,1)$ otherwise, and scale $a_{ji}^{\prime\prime}(\tau)$ by $\frac{b_{ji}'}{\sum_\tau a_{ji}^{\prime\prime}(\tau)}$. We will use $\{b_{ji}'\forall j\neq i\}$ as the cartesian coordinates in our $d$-ball. The full algorithm is given below.

\begin{algorithm}[H]\caption{UUMC SVAR Generation}\label{ts_alg}
\setlength{\abovedisplayskip}{4pt}
\setlength{\belowdisplayskip}{3pt}

\begin{algorithmic}
    \Require $\mathcal G=\left(\{\hat{X}_i(\tau)\}, E\right)$ a time series graph with topological order \\ 
    $i \in 1...N$ and adjacency matrix \\
    $E = (e_{ji}(\tau)) \in {\{0,1\}}^{N\times N\times (\tau_{\max}+1)}$ with $N \in \mathbb N$ and $\tau_{\max}\in\mathbb Z^+$
    \Ensure  $j\geq i \Rightarrow e_{ji}(0)=0$
\end{algorithmic}
    $R = (\rho_{ji}(\tau)); \; \rho_{ii}(0) \gets 1 \; \forall i$\; \\
    $\vec{\hat{s}} = (\hat{s}_i) \gets \mathbf{1}_N$\; \\
    $E_S = (h_{ji}) \gets ||_\tau E$ is the adjacency matrix for the summary graph \\
    $B' = (b'_{ji}) \gets 0_{N,N}$ \\
    $\vec{C} = (C_i) \gets \boldsymbol{1}_N$ \\
    draw $\Delta \sim \mathcal F(100,100)$\\
    \For{$i \in 1...N$}{
        $d_i \gets \sum_{j\neq i} \vec{h}_i = \#pa_{E_S}(\hat{X}_i)\setminus \hat{X}_i$ \\ 
        \If{$d_i>0 || h_{ii}\neq 0$}{
            draw $\vec{a_{i}^{\prime\prime}}(\tau)\sim \mathcal N(0,1)^{N\times  (\tau_{\max}+1)}$\\
            $a_{ji}^{\prime\prime}(\tau) \gets a_{ji}^{\prime\prime}(\tau)e_{ji}(\tau) \; \forall j, \tau$ 
        }
        \If{$h_{ii}\neq 0$}{
            draw $b_{ii}'\sim \mathcal U(0,1)$\\
            $b_{ii}\gets (b_{ii}')^\Delta$\\
            $\hat{a}_{ii}(\tau) \gets a_{ii}'(\tau) \gets a_{ii}^{\prime\prime}(\tau)\frac{b_{ii}}{\sum_\tau a_{ii}^{\prime\prime}(\tau)}$\\
            $s_i' \gets \sqrt{1-b_{ii}^2}$
        }
        \If{$d_i>0$}{
            draw $b_{j\neq i, i}^{\prime\prime} \sim \mathcal N(0,1)$; $r_i \sim \mathcal U^{\sfrac{1}{d_i}}(0,1)$ \\
            $b_{j\neq i,i}'\gets \frac{r_is_i'}{\sqrt{\sum_{j\neq i}(b_{ji}')^2}}b_{ji}^{\prime\prime} \\
            s_i'\gets s_i'\sqrt{1-r_i^2}$\\
            $b_{j\neq i, i} \gets b_{j\neq i, i}' \frac{b_{jj}^\Delta - b_{ii}^\Delta}{b_{jj}-b_{ii}}$\\
            $a_{j\neq i,i}'(\tau) \gets a_{j\neq i,i}^{\prime\prime}(\tau)\frac{b_{j\neq i,i}}{\sum_\tau a_{j\neq i,i}^{\prime\prime}(\tau)}$\\
            $C_i \gets C_i; \; \hat{s}_i \gets s_i'/C_i$
        }
    }
    $\hat{a}_{ji}(\tau) \gets a_{ji}'(\tau)/C_i \; \forall i,\tau,j\neq i$\\
    Solve for $C_i$ and $\rho_{ji}(\tau)$ in 
    \begin{align*}
         \begin{cases}
             1 = \sum_{jk}\sum_{\tau\nu} a_{ji}'(\tau)a_{ki}'(\nu)\rho_{jk}(\tau-\nu) + (s_i')^2 \; \forall i\\
             \rho_{ji}(\tau) = \sum_k\sum_\nu a_{ki}'(\nu)\rho_{jk}(\tau-\nu) \; \forall i,j, \tau
         \end{cases}
    \end{align*} 
    Check that \(\det(\hat{A}(0)^{-1} + \mathbf{I}_N - \sum_{\tau=1}^{\tau_{\max}}\hat{A}(\tau)z^\tau)=0\) yields roots $z$ that lie in the unit circle.
\end{algorithm}
\end{document}